\documentclass[english]{article}

\usepackage{geometry}
\usepackage[T1]{fontenc}
\usepackage[latin9]{inputenc}
\usepackage{bm}
\usepackage{amsmath,mathtools}
\usepackage{amssymb}
\usepackage[unicode=true,
 bookmarks=false,
 breaklinks=false,pdfborder={0 0 1},colorlinks=false]
 {hyperref}
\hypersetup{colorlinks,citecolor=blue,filecolor=blue,linkcolor=blue,urlcolor=blue}

\makeatletter
\usepackage{amsthm}
\usepackage{cite}  
\usepackage{comment}
\usepackage[sort,compress]{natbib}
\usepackage{booktabs}
\usepackage{mathabx}
\usepackage{graphicx}

\usepackage[linesnumbered,ruled,vlined]{algorithm2e}

\SetCommentSty{mycommfont}
\usepackage{algorithmic}

\usepackage{float}
\usepackage{multirow}
\usepackage{dsfont}
\usepackage{tcolorbox}
\usepackage{color}
\definecolor{yxc}{RGB}{255,0,0}
\definecolor{ytw}{RGB}{255,69,0}
\definecolor{gen}{RGB}{0,0,200}
\definecolor{zhh}{RGB}{200,200,0}

\allowdisplaybreaks

\renewcommand{\P}{\mathbb{P}}

\newcommand{\R}{\mathbb{R}}

\newcommand{\B}{\mathcal{B}}
\newcommand{\T}{\mathcal{T}}

\newcommand{\tr}{\mathrm{Tr}}

\newcommand{\M}{\mathcal{M}}

\renewcommand{\d}{\mathrm{d}}

\newcommand{\KL}{\mathsf{KL}}

\newcommand{\Pdata}{p_{\mathsf{data}}}

\newcommand{\E}{\mathbb{E}}

\newcommand{\N}{\mathcal{N}}
\newcommand{\m}{\,||\,}
\newcommand{\F}{\mathrm{F}}
\newtheorem{definition}{Definition}

\theoremstyle{plain} \newtheorem{lemma}{\textbf{Lemma}}\newtheorem{theorem}{\textbf{Theorem}}

\theoremstyle{assumption}\newtheorem{assumption}{\textbf{Assumption}}
\theoremstyle{remark}\newtheorem{remark}{\textbf{Remark}}
\theoremstyle{Corollary}\newtheorem{corollary}{\textbf{Corollary}}

\allowdisplaybreaks

\title{Denoising diffusion probabilistic models are \\ optimally adaptive to unknown low dimensionality}

\author{%
	Zhihan Huang\thanks{Department of Statistics and Data Science, the Wharton School, University of Pennsylvania; email: \texttt{\{zhihanh,ytwei,yuxinc\}@wharton.upenn.edu}.}
	 \and
Yuting Wei\footnotemark[1] 
\and
 Yuxin Chen\footnotemark[1]
}

\date{October 2024;~~Revised: Feburary, 2026}

\makeatother

\begin{document}

\maketitle 

\begin{abstract}
The denoising diffusion probabilistic model (DDPM) has emerged as a mainstream generative model in generative AI. While sharp convergence guarantees have been established for the DDPM, the iteration complexity is, in general,  proportional to the ambient data dimension, resulting in overly conservative theory that fails to explain its practical efficiency. This has motivated the recent work \cite{li2024adapting} to investigate how the DDPM can achieve sampling speed-ups through automatic exploitation of intrinsic low dimensionality of data.  We strengthen this line of work by demonstrating, in some sense, optimal adaptivity to  unknown low dimensionality. For a broad class of data distributions with intrinsic dimension $k$, we prove that the iteration complexity of the DDPM scales nearly linearly with $k$, which is optimal when using KL divergence to measure distributional discrepancy. 
%
\end{abstract}

\setcounter{tocdepth}{2}
\tableofcontents

\section{Introduction}
\label{sec:intro}


In generative modeling, one is asked to produce novel samples that resemble, in distribution, the data on which the model has been trained. In order to fully unleash its power, a generative modeling algorithm must exploit, either implicitly or explicitly, the distinctive features of the data distribution of interest, rather than merely optimizing for worst-case scenarios.  
A prominent example of such distinctive features is  the low-dimensional structure underlying the data of interest,  which permeates modern statistical modeling. For instance, image data often resides on some low-dimensional manifold \citep{simoncelli2001natural,pope2021intrinsic}, and genomics data might also be effectively modeled through manifolds \citep{moon2018manifold,zhu2018reconstructing}. 
In fact, the ImageNet dataset with ambient dimension $d \approx 10^4$ is estimated to have an intrinsic dimension of $k \approx 43$ \citep{pope2021intrinsic}, which is significantly smaller than the ambient dimension.
While harnessing these low-dimensional data structures is generally believed to enable more efficient data generation, 
the theoretical underpinnings for this capability in contemporary generative modeling algorithms remain vastly under-explored.

In this paper, we study how low-dimensional data structure influences the performance of diffusion models, an emerging paradigm that facilitates remarkable progress across various subfields of modern generative modeling \citep{chen2024overview,sohl2015deep,dhariwal2021diffusion,ho2022video,ramesh2022hierarchical,watson2023novo}. 
Empirical evidence suggests that diffusion models might not suffer from the curse of dimensionality.
Some of the empirical methods explicitly exploit dimension reduction designs like latent diffusion to leverage low dimensionality \citep{rombach2022high},  while in the meantime vanilla diffusion models also seem to exhibit adaptivity to unknown low-dimensional structure in practice. However, a rigorous understanding of the efficacy of standard diffusion models in the presence of low-dimensional structure remains elusive. 
Consider, for concreteness, a mainstream diffusion model known as the {\em denoising diffusion probabilistic model (DDPM)} \citep{ho2020denoising}, which has received widespread adoption in practice and sparked considerable interest from the theoretical community to demystify its efficacy (e.g., \citet{chen2022sampling,chen2022improved,lee2022convergence,lee2023convergence,li2024towards,tang2023diffusion,mbacke2023note,tang2024score,li2025dimension,liang2024non,li2024d}).  In order for the DDPM sampler to generate a $d$-dimensional new sample that approximates well the target distribution  in terms of KL divergence, 
the number of steps\footnote{Given that each step of the DDPM requires evaluating a score function (typically through computing the output of a large neural network or transformer), the sampling speed of the DDPM is largely dictated by the number of steps. } needs to scale linearly with $d$, a scaling that has been shown to be un-improvable in general \citep{benton2024nearly}. 
However, in application like image generation, the number of pixels in each image might oftentimes be exceedingly large, far exceeding the number of steps typically needed to run the DDPM in practice. This indicates that even linear scaling with $d$ does not adequately explain the observed efficiency of the DDPM in practice.

To bridge the theory-practice gap for the data generation stage,  
the seminal work by \citet{li2024adapting} demonstrated that DDPM can automatically adapt to the intrinsic dimension of the target distribution $\Pdata$, 
which rigorously justifies sampling speed-ups if the support of $\Pdata$ is intrinsically low-dimensional. 
As a particularly intriguing message of this prior result, the DDPM sampler, in its original form, achieves favorable adaptivity without explicitly modeling the low-dimensional structure.

Despite this remarkable theoretical progress, however, 
it was unclear to what extent the low-dimensional structure could enhance the sampling efficiency of the DDPM. More precisely, \citet{li2024adapting} proved that when the support of $\Pdata$ has intrinsic dimension $k$ (to be made precise in Section~\ref{sec:low-dim}),  the iteration complexity of the DDPM scales at most with $k^4$ (up to some logarithmic factor and other dependency on the target accuracy level), provided that KL divergence is used to measure the distributional discrepancy. 
This result was recently improved by \citet{azangulov2024convergence} 
to the order of $k^3$,  assuming that data lies within some low-dimensional smooth manifolds. 
This raises a natural question: is the $k^3$ scaling essential, and if not, how to sharpen it? Addressing this question would provide a quantitative understanding about DDPM's degree of adaptivity when tackling low-dimensional structure.

\paragraph{Main contributions.} 
In this work, we develop a sharp convergence theory for the DDPM in the face of unknown low dimensionality. For a broad family of target data distributions $\Pdata$ with intrinsic dimension $k$, 
we show that it takes DDPM at most $O(k/\varepsilon^2)$ (up to some logarithmic factor) steps to yield an $\varepsilon^2$-accurate distribution --- measured by the KL divergence --- assuming access to perfect score functions. 
This result, which shares similarity with a concurrent work by \citet{potaptchik2024linear} to be discussed momentarily,  improves upon the iteration complexity derived by  \citet{azangulov2024convergence} by a factor of $k^2$, and achieves nearly optimal scaling in $k$ without any burn-in requirement. 
Further, our convergence theory can be easily extended to accommodate imperfect score estimation.


To provide some intuition about its adaptivity,  
it is helpful to recognize that the original DDPM update rule by \citet{ho2020denoising} is precisely equivalent
 to running a carefully parameterized stochastic differential equation (SDE) upon discretization using the exponential integrator, 
 a key observation that appeared in the recent work by \citet{azangulov2024convergence}. 
 Crucially, the drift term of this discretized SDE is semi-linear, with the non-linear component proportional to the posterior mean of the data given its Gaussian corruption. Intuitively, when the support of $\Pdata$ is low-dimensional, this nonlinear drift term acts as a ``projection'' onto the low-dimensional manifold of interest,  harnessing this intrinsic structure to enhance the smoothness of the solution path and thereby speed up the sampling process. 
 Somewhat surprisingly, 
while DDPM was initially proposed to maximize a variational lower bound on the log-likelihood \citep{ho2020denoising}, the resulting update rule  inherently adopts the desirable time discretization schedule, in a way that is fully adaptive to unknown low-dimensional distributions. 
A key ingredient underlying our analysis of the discretized SDE lies in carefully calibrating the evolution of the posterior mean function along the forward process.

\paragraph{The concurrent work \citet{potaptchik2024linear}.}
While finalizing this work, we became aware of an interesting concurrent work \cite{potaptchik2024linear}, posted two weeks before our manuscript, that delivered the first convergence guarantees exhibiting linear dependency on the intrinsic manifold dimension of the data. 
Our paper, which stands as an independent and concurrent contribution, has similar aim and establishes similar results as \cite{potaptchik2024linear},  albeit under different assumptions on the low-dimensional structure.  We shall provide more detailed comparisons between our work and  \cite{potaptchik2024linear} in Section~\ref{sec:convergence-DDPM}.

\paragraph{Notation.} 
We introduce several notation to be used throughout. For any two functions $f$ and $g$, we adopt the notation $f \lesssim g$ or $f = O(g)$ (resp.~$f \gtrsim g$ or $f = \Omega(g)$) to mean that there exists some universal constant $C > 0$ such that $f \le C g$ (resp.~$f \ge Cg$). The notation $f \asymp g$ or $f = \Theta(g)$ then indicates that $f \lesssim g$ and $f \gtrsim g$ hold at once. Additionally,  $f = \widetilde{O}(g)$ (resp.~$f = \widetilde{\Theta}(g)$) is defined analogously except that the logarithmic dependency is hidden. 
%
We denote by $\|\cdot\|_\F$ the Frobenius norm of a matrix, and use $\|\cdot\|_2$ to denote the Euclidean norm of a vector. The notation $X \overset{\mathrm d}{=} Y$ indicates that the random objects $X$ and $Y$ are identical in distribution.


\section{Preliminaries}
\label{sec:background}


In this section, we review some basics about the original DDPM sampler, 
introduce a useful perspective from SDEs, and describe some fundamental properties about the (Stein) score function. 

\paragraph{The DDPM sampler.}  
The procedure of the DDPM sampler is described in the sequel. 
\begin{itemize}
	\item {\em Forward process.} Consider a forward process $(X_t)_{t\in [0,T]}$ generated such that
%
\begin{equation}
	X_{t} \overset{\mathrm{d}}{=} \sqrt{1-\sigma_{t}^{2}}X_{0}+\sigma_{t}W\qquad
	\text{with }\sigma_{t} \coloneqq \sqrt{1-e^{-2t}},
	\label{eq:forward-marginal}
\end{equation}
where $X_{0}\sim \Pdata$ and $W\sim\mathcal{N}(0,I_{d})$ are independently generated. In words, this forward process starts with a realization of $X_0$ drawn from the target data distribution $p_{\mathsf{data}}$, and injects increasingly more Gaussian noise as $t$ grows. Evidently, the distribution of $X_T$ becomes exceedingly close to $\mathcal{N}(0,I_d)$ for a large enough time horizon $T$. 

	\item {\em Data generation process.} In order to operate in discrete time, let us choose a set of $N$ time discretization points $0=t_0<t_1<\dots<t_N<t_{N+1}=T$ for some large enough $T$, and adopt the update rule below:\footnote{While the DDPM update rule was introduced in \citet{ho2020denoising} through the choices of $\{\alpha_{n}\}$ instead of the time discretization points $\{t_n\}$, it is equivalent to the version adopted herein. }   
\begin{align}
\begin{split}
	Y_{t_{0}} & \sim\mathcal{N}(0,I_{d}),\\
	Y_{t_{n+1}} & =\frac{1}{\sqrt{\alpha_{n}}}\big(Y_{t_{n}}+(1-\alpha_{n})\widehat{s}_{T-t_{n}}(Y_{t_{n}})\big)+\sqrt{\frac{(1-\alpha_{n})(1-\overline{\alpha}_{n+1})}{1-\overline{\alpha}_{n}}}Z_{n},
	\qquad n=0,\dots,N-1,
\end{split}
\label{eq:DDPM-update}
\end{align}
where  $\widehat{s}_{t}(\cdot)$ represents an estimate of the true score
function $s_{t}(\cdot) = \nabla\log p_{X_{t}}(\cdot)$ w.r.t.~the forward process \eqref{eq:forward-marginal}, 
and the $Z_{n}$'s are random vectors independently drawn from $Z_{n}\sim\mathcal{N}(0,I_{d})$. 
Here and throughout, the coefficients $\{\alpha_n\}$ and $\{\overline{\alpha}_n\}$ are taken to be
\begin{align}
	\alpha_n \coloneqq e^{-2(t_{n+1}-t_n)} \qquad \text{and} \qquad
	\overline{\alpha}_n \coloneqq \prod_{i=n}^N \alpha_i = e^{-2(T-t_n)},
	\label{eq:alpha-bar-n}
\end{align}
which are fully determined by the locations of the time discretization points. 
In a nutshell, each iteration in \eqref{eq:DDPM-update} computes the weighted superposition of the current iterate, its corresponding score estimate, and an independent Gaussian vector, 
		where the weights were originally proposed by \citet{ho2020denoising} to optimize a variational lower bound on the log-likelihood.

\end{itemize}

%

\paragraph{An SDE perspective.}  
The feasibility to reverse the forward process with the aid of score functions can be elucidated in the continuous-time limits, 
an approach popularized by \citet{song2020score}.  
To be more specific,  construct the forward process as the standard Ornstein-Uhlenbeck (OU) process \citep{oksendal2003stochastic}
\begin{align}
    \d X_t &= -X_t \d t + \sqrt{2} \,\d W_t, \quad X_0 \sim \Pdata, \quad  t \in [0,T] \label{eq:forward}
\end{align}
with $(W_t)$ the standard Brownian motion in $\mathbb{R}^d$.  
It is well-known this process \eqref{eq:forward} satisfies property \eqref{eq:forward-marginal}.  
Crucially, process \eqref{eq:forward} can be reversed through the following SDE with the aid of score functions: 
\begin{align}
    \d Y_t &= \big(Y_t + 2s_{T-t}(Y_t)\big)\d t + \sqrt{2} \,\d B_t,  \quad t \in [0,T], \label{eq:backward}
\end{align}
where $(B_t)$ represents another standard Brownian motion in $\mathbb{R}^d$ independent from $(X_t)$ and $(W_t)$. 
If $Y_0 \sim X_T$, then classical results in the SDE literature \citep{anderson1982reverse,haussmann1986time}  tell us that 
\begin{align}
Y_{T-t} \stackrel{\mathrm{d}}{=} X_t,\qquad 0\leq t\leq T
\label{eq:equiv-Y-X-reverse}
\end{align}
holds under fairly mild regularity conditions on $\Pdata$, thereby confirming the plausibility of score-based generative models. 
Note, however, that practical implementation of \eqref{eq:backward} has to operate in discrete time, 
and the sampling fidelity relies heavily upon the time discretization error.  
In fact, a large body of prior convergence theory for DDPM-type stochastic samplers (e.g., \citet{chen2022sampling,chen2023improved,benton2024nearly}) was established by studying a certain time-discretized version of \eqref{eq:backward}, with $s_{T-t}(\cdot)$ replaced by the score estimate $\widehat{s}_{T-t}(\cdot)$ to account for the effect of score estimation errors.

\paragraph{Score functions.} 
Before proceeding, let us take a moment to delve into the  score function defined as
\begin{equation}
	\label{eq:defn-score-function}
	s_t(x) \coloneqq \nabla \log q_{t}(x),
\end{equation}
where $q_{t}(\cdot)$ reprsents the marginal density function of $X_t$ w.r.t.~the Lebesgue measure.
By defining the posterior mean and covariance of $X_0$ given $X_t$ as follows:
\begin{align}
    \mu_t(x) \coloneqq \E[X_0 \mid X_t = x] \quad\text{ and }\quad \Sigma_t(x) \coloneqq \mathrm{Cov}[X_0 \mid X_t = x],\notag
\end{align}
one can readily derive from Tweedie's formula \citep{efron2011tweedie} that
\begin{align}
    \mu_t(X_t) = \frac{1}{\sqrt{1-\sigma_t^2}}\left(X_t + \sigma^2_t \nabla \log q_{t}(X_t)\right) = \frac{1}{\sqrt{1-\sigma_t^2}}\left(X_t + \sigma^2_t s_t(X_t)\right),\label{eq:tweedie}
\end{align}
where we recall that $\sigma_t=\sqrt{1-e^{-2t}}$. 
In other words, the score function is intimately connected with the posterior mean function. 
Throughout the paper, we shall denote by $\widehat{\mu}_t(\cdot)$ the estimate of the posterior mean $\mu_t(\cdot)$, 
which allows us to express the score estimate $\widehat{s}_t(\cdot)$ as 
\begin{align}
	\widehat{\mu}_t(X_t) = \frac{1}{\sqrt{1-\sigma_t^2}}\left(X_t + \sigma^2_t \widehat{s}_t(X_t)\right). 
	\label{eq:ep-mu-s}
\end{align}

%

\section{DDPM as an adaptively discretized SDE}
\label{sec:DDPM-SDE}

While previous studies (e.g., \citet{chen2022sampling,chen2023improved,benton2024nearly}) have explored the convergence properties of time discretization of the ideal reverse process \eqref{eq:backward} through the SDE framework, 
a dominant fraction of these results fell short of directly addressing the precise update rule \eqref{eq:DDPM-update} of the DDPM sampler. 
As pointed out by \citet{li2024adapting}, the coefficients (or weights) adopted in the DDPM-type sampler are critical for effective adaptation to unknown low dimensionality; in fact, even mild perturbation of the coefficients can lead to considerable performance degradation, a curious phenomenon that does not occur in the general full-dimensional setting.

In order to rigorously pin down the effectiveness of the DDPM update rule \eqref{eq:DDPM-update}, 
a key step is to recognize its equivalence to a carefully constructed SDE. To be precise, set 
\begin{align}
\label{eq:defn-eta-t}
\eta_t \coloneqq \frac{\sqrt{1-\sigma_t^2}}{\sigma_t^2} = \frac{e^{-t}}{1-e^{-2t}}
\end{align}
for notational convenience,  and let us introduce the following SDE~\eqref{eq:ddpm-SDE-adaptive} that is provably equivalent to the DDPM update rule.  
\begin{itemize}
\item {\bf An adaptively discretized reverse SDE.} Initialized at $Y_0 \sim \mathcal{N}(0,I_d)$, 
this SDE proceeds as
\begin{subequations}
\label{eq:ddpm-SDE-adaptive}
    \begin{align}
        \d Y_t &= \left( Y_t + 2\left(\frac{\eta_{T-t}}{\eta_{T-t_n}}\cdot\frac{1}{\sigma^2_{T-t_n}}Y_{t_n} - \frac{1}{\sigma^2_{T-t}}Y_t\right) +  \frac{2\eta_{T-t}}{\eta_{T-t_n}}\widehat{s}_{T-t_n}(Y_{t_n})\right)\d t + \sqrt{2}\,\d \B_t\label{eq:ddpm-s}\\
        &= \left(\left(1-\frac{2}{\sigma_{T-t}^2}\right)Y_t + \frac{2\sqrt{1-\sigma_{T-t}^2}}{\sigma^2_{T-t}}\widehat{\mu}_{T-t_n}(Y_{t_n})\right)\d t + \sqrt{2}\,\d B_t
        \label{eq:ddpm-mu}
    \end{align}
    in the $n$-th ($0\leq n< N$) time interval $t\in [t_{n},t_{n+1})$, where $(B_t)$ stands for a standard Brownian motion in $\mathbb{R}^d$.  
    Here, the equivalence of \eqref{eq:ddpm-s} and \eqref{eq:ddpm-mu} follows immediately from \eqref{eq:ep-mu-s}.
\end{subequations}
    
\end{itemize}

\begin{remark}
After the initial posting of this paper, the authors of the concurrent work \cite{potaptchik2024linear} informed us 
that a similar SDE already appeared in the recent work \citep[Section 7.2]{azangulov2024convergence} as a type of backward SDE discretization with first-order score correction, where they discussed the motivation of score correction and its effect upon the discretization error. 
See also \citet[Page~90]{azangulov2024convergence} for their justification of the SDE equivalence of the DDPM. 
\end{remark}





Importantly, SDE~\eqref{eq:ddpm-SDE-adaptive} is a continuous interpolation of the DDPM update rule~\eqref{eq:DDPM-update}, which means that solving the above SDE~\eqref{eq:ddpm-SDE-adaptive} leads to exactly the original DDPM sampler, as asserted by the following theorem.
\begin{theorem}\label{thm:ddpm}
    Let $(Y_t)_{t\in [0,T]}$ be the solution to the SDE~\eqref{eq:ddpm-SDE-adaptive}. 
    Then for any $n=0,\ldots,N-1$, one can write 
    \begin{align*}
	Y_{t_{n+1}} & =\frac{1}{\sqrt{\alpha_{n}}}\big(Y_{t_{n}}+(1-\alpha_{n})\widehat{s}_{T-t_{n}}(Y_{t_{n}})\big)+\sqrt{\frac{(1-\alpha_{n})(1-\overline{\alpha}_{n+1})}{1-\overline{\alpha}_{n}}}\widetilde{Z}_{n}
\end{align*}
    for statistically independent random vectors $\{\widetilde{Z}_{n}\}$ in $\mathbb{R}^d$ obeying $\widetilde{Z}_{n}\sim \mathcal{N}(0,I_d)$, which recovers the DDPM sampler~\eqref{eq:DDPM-update}.
\end{theorem}
\noindent 
The proof of Theorem~\ref{thm:ddpm} is postponed to Appendix~\ref{sec:proof-prop:ddpm}.




Recall that in the general full-dimensional settings, 
previous work (e.g., \citet{chen2023improved}) suggested that discretizing different parametrization of \eqref{eq:backward} typically results in the same order of discretization errors, indicating that distinguishing between them might be unnecessary. This message, however, does not hold in the presence of low-dimensional structure, which merits further discussions as follows.


%
\begin{itemize}
    \item {\em A reparameterized semi-linear SDE.} As can be readily seen, if one has access to the true score function (i.e., $\widehat{s}_{t}(\cdot) = s_{t}(\cdot)$), then for some $\delta > 0$, SDE~\eqref{eq:ddpm-mu} can be viewed as the discretization of the true reverse SDE in the time horizon $[0,T-\delta]$ under the exponential integrator scheme:
\begin{align}
    \d Y_t &= \big(Y_t + 2s_{T-t}(Y_t)\big)\d t + \sqrt{2}\, \d B_t\notag\\
    &= \Bigg(\left(1-\frac{2}{\sigma_{T-t}^2}\right)Y_t + \frac{2\sqrt{1-\sigma_{T-t}^2}}{\sigma^2_{T-t}}\mu_{T-t}(Y_{t})\Bigg)\d t + \sqrt{2}\,\d B_t, \label{eq:ddpm-mu-oracle}
\end{align}
where the last line invokes the Tweedie formula \eqref{eq:tweedie}.
Here, the stopping time $\delta$ is introduced to avoid the singularity of the drift term when $t$ approaches $T$.
In other words, the key is to rearrange the drift term of the reverse process \eqref{eq:backward} into a different semi-linear form, with the nonlinear component reflecting the posterior mean rather than the score. 

    \item {\em Adaptation to data geometry.} 
    To elucidate why the semi-linear structure in \eqref{eq:ddpm-mu-oracle} favors low-dimensional data structure, note that the discretization error depends heavily upon the smoothness of the non-linear drift term of the reparameterzied semi-linear SDE along the solution path.  In \eqref{eq:ddpm-mu-oracle}, the non-linear drift term is proportional to the posterior mean of $X_0$ given $X_{T-t}$; informally, the posterior mean term acts as some sort of projection operator onto the support of the target data, 
    which effectively exploits the (unknown) low-dimensional data geometry to enhance the smoothness of the solution path.

    \item {\em Parameterization matters.} 
    It is worth noting that 
    how one re-parameterizes the drift term of \eqref{eq:backward} --- namely, how one rearranges the semi-linear drift term therein --- can exert significant influences upon the sampling efficiency. 
    This can be elucidated via the inspiring example 
    singled out in \citet[Theorem~2]{li2024adapting}. More concretely,  when the target data distribution is isotropic Gaussian on a $k$-dimension linear subspace ($k\ll d$), the iteration complexity of the DDPM sampler~\eqref{eq:DDPM-update} is proportional to $O(\mathrm{poly}(k))$ \citep{li2024adapting}, 
    but most of other DDPM-type update rules (i.e., the ones with different coefficients than the original DDPM sampler) can take  $\Omega(\mathrm{poly}(d))$ steps to yield the target sampling quality. 
    Informally speaking, 
    parameterizing the semi-linear drift term differently might result in a solution within $O(d/\mathrm{poly}(N))$ KL divergence from the DDPM solution. \footnote{Roughly speaking, in view of the Taylor expansion, we know that the discrepancy between different solutions is of order $O(1/\mathrm{poly}(N))$; but since $d$-dimensional Gaussian noise is added in each step of the sampling procedure, this discrepancy can result in $O(d/\mathrm{poly}(N))$ difference in KL divergence per step.} 
		Note that this discrepancy is negligible in the full-dimensional setting, but could dominate the sampling error when $k$ is exceedingly small compared to $d$ and when $N$ is approximately proportional to $k$.\footnote{This order of $N$ coincides with the number of iterations we suggest in Corollary~\ref{cor:main} for the DDPM sampler (up to logarithmic factor).}     
    All this underscores the critical importance of identifying a parameterization that adapts naturally to the intrinsic data geometry.




\end{itemize}

\section{Main results}
\label{sec:main-results}



Armed with the equivalence between the original DDPM sampler and the SDE~\eqref{eq:ddpm-SDE-adaptive} (cf.~Theorem~\ref{thm:ddpm}), 
we can proceed to establish a sharp convergence theory for the DDPM sampler, leveraging the toolbox from the SDE literature. 
Before proceeding to our main theory, let us begin by imposing a couple of key assumptions.

\subsection{Assumptions and key quantities}

\subsubsection{Low dimensionality of the target distribution}\label{sec:low-dim}

To introduce the intrinsic dimensionality of the target data distribution $\Pdata$, 
we take a look at how complex the support of $\Pdata$ is. Here and throughout, we denote by 
$\mathcal{X}_{\mathsf{data}} \in \R^d$  the support of  $\Pdata$, i.e., the closure of the intersection of all the sets $\mathcal{X}' \in \R^d$ such that $\P_{X_0\sim \Pdata}(X_0 \in \mathcal{X}') = 1$.

Given that we focus on data residing in the Euclidean space, 
we can measure the complexity through the Euclidean metric entropy below defined in \citet[Chapter 5]{wainwright2019high}:

\begin{definition}[Covering number and metric entropy]
    For any set $\mathcal{M}\subseteq \mathbb{R}^d$,  
    the Euclidean covering number at scale $\varepsilon_0 > 0$, denoted by $N^{\mathsf{cover}}(\mathcal{M},\|\cdot\|_2,\varepsilon_0)$,  is the smallest $s$ such that there exist points $x_1, \dots , x_s$ obeying
    \begin{align*}
        \mathcal{M} \subseteq \bigcup_{i=1}^s B(x_i,\varepsilon_0),
    \end{align*}
    where $B(x_i,\varepsilon_0) \coloneqq \{x \in \mathbb{R}^d \mid \|x-x_i\|_2 \leq \varepsilon_0\}$. The metric entropy 
    is defined as $\log N^{\mathsf{cover}}(\mathcal{M},\|\cdot\|_2,\varepsilon_0)$.
\end{definition}

Following \cite{li2024adapting}, we make the following assumption to characterize the complexity of the data distribution, which generalizes the data generation assumption in \citet{li2024adapting}.
We shall often refer to $k$ as the intrinsic dimension of $\mathcal{X}_{\mathsf{data}}$. 
\begin{assumption}[Intrinsic dimension]\label{ass:low-dim}
    Suppose that $C_0>0$ is some sufficiently large universal constant. For $\varepsilon_0 = k^{-C_0}$, the metric entropy of the support $\mathcal{X}_{\mathsf{data}}$ of the data distribution $\Pdata$ satisfies
    \begin{align*}
        \log N^{\mathsf{cover}}(\mathcal{X}_{\mathsf{data}},\|\cdot\|_2,\varepsilon_0) \le C_{\mathsf{cover}} k \log \frac{1}{\varepsilon_0}
    \end{align*}
    for some universal constant $C_{\mathsf{cover}}>0$.
\end{assumption}



Next, we adopt a fairly mild assumption regarding the boundedness of the support $\mathcal{X}_{\mathsf{data}}$, allowing the size of the data to scale polynomially in $k$. 
\begin{assumption}\label{ass:bounded}
    Suppose that there exists some universal constant $C_R > 0$ such that
    \begin{align*}
        \sup_{x \in \mathcal{X}_{\mathsf{data}}} \|x\|_2 \le R, \qquad \mathrm{where} \enspace  R \coloneq k^{C_R}.
    \end{align*}
\end{assumption}
\begin{remark}Note that in Assumption~\ref{ass:bounded},  $C_R$ is allowed to be arbitrarily large, as long as it is a numerical constant. This means that $R$ can be an arbitrarily large polynomial in $k$. \end{remark}

To help interpret the above assumptions, 
we single out several important examples below.

%

\begin{itemize}
\item  {\em Linear subspace.} The metric entropy of a bounded subset of any $k$-dimensional linear subspace is at most $O(k\log (1/\varepsilon_0))$ \citep[Section 4.2.1]{vershynin2018high}, and hence  Assumption~\ref{ass:low-dim} covers linear subspaces studied by some prior work (e.g., \citet{chen2023score}). 
%

\item {\em Low-dimensional manifold.} 
Assumption~\ref{ass:low-dim} covers various manifolds under mild regularity conditions, which have been commonly used in prior work to model low-dimensional structure. 
To make it precise, let us first introduce some geometric quantities of the manifold, whose geometric interpretations can be found in \citet[Chapter 6]{lee2018introduction}. 
Suppose $\M \subseteq \R^d$ to be an isometrically embedded, compact, $k$-dimensional submanifold, then we can define:  
\begin{definition}\label{def:reach}
    The reach of $\M$, denoted by \textcolor{black}{$\tau_\M$}, is the distance between $\M$ and $\mathrm{Med}(\M)$,  where 
    \begin{align*}
        \mathrm{Med}(\M) \coloneqq \big\{x \in \R^d \mid \mathrm{there} \,\, \mathrm{exists} \,\, p \neq q \in \M  \,\, \mathrm{such}  \,\, \mathrm{that} \,\,\|p-x\|_2 = \|q-x\|_2 = \mathsf{dist}(x,\M) \big\},
    \end{align*}
and $\mathsf{dist}(x,\M)$ denotes the Euclidean distance of $x$ to $\M$. Also, $\mathsf{vol}(\M)$ is the volume, or equivalently, the $k$-dimensional Hausdorff measure, of the manifold $\M$.
\end{definition}
\begin{figure}[H]
    \centering
    \includegraphics[width=0.5\textwidth]{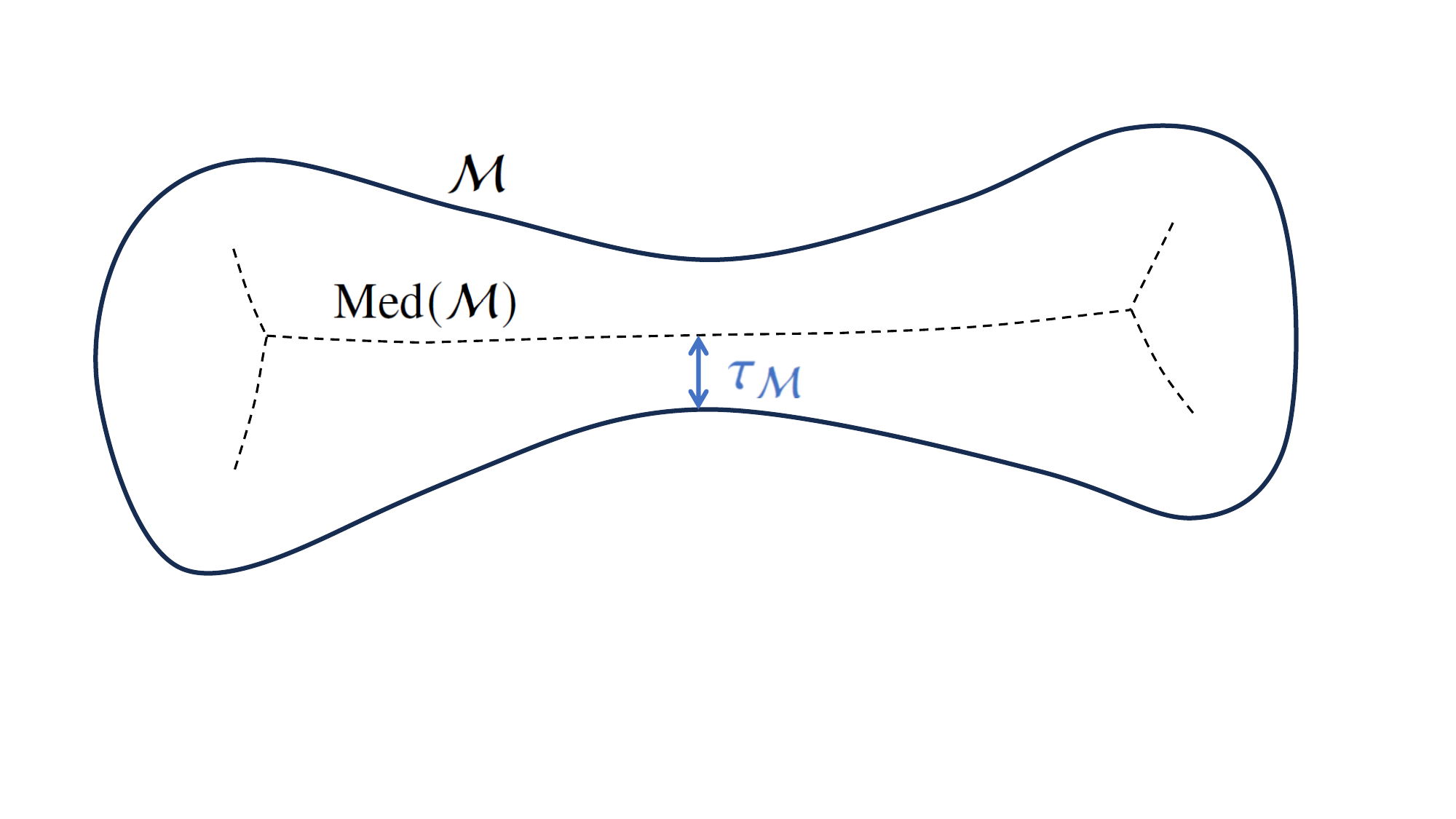}
    \caption{Illustration of reach $\tau_{\M}$ of a manifold $\M$ embedded in $\mathbb{R}^2$.}
    \label{fig:reach-diagram}
\end{figure}
Roughly speaking, the reach of a manifold measures the scale of which the manifold can be regarded as a linear subspace.
An illustration of the reach is provided in Figure~\ref{fig:reach-diagram}.
The larger the reach $\tau_\M$ is, the more regular the manifold appears.
The volume of the manifold $\M$ is a measure of the size of the manifold, which generalizes the concept of volume in 3-dimensional Euclidean space.
Instead of requiring data to be strictly supported on the manifold, Assumption~\ref{ass:low-dim} also allows the data to lie in the $\varepsilon_0$-enlargement of the manifold $\M$ defined as 
$\M^{\varepsilon_0} := \{x \in \R^d \mid \mathsf{dist}(x,\M) \le \varepsilon_0\}$.
This is stated formally in the following lemma, with the proof deferred to Appendix~\ref{sec:proof-prop:manifold}. 
%
\begin{lemma}\label{prop:manifold}
    If the support of the target data distribution $\Pdata$ is contained within the manifold $\M^{\varepsilon_0}$, i.e., $\mathcal{X}_{\mathsf{data}} \subseteq \M^{\varepsilon_0}$,and if 
    $\tau_\M = \Omega(k^{-C_1})$ and $\mathsf{vol}(\M) = O(k^{C_2})$ for some universal constants $C_1,C_2>0$, then Assumption~\ref{ass:low-dim} holds.
\end{lemma}

\item {\em A general set with doubling dimension $k$.} Another notion for intrinsic dimension used in statistics for a general set in $\mathbb{R}^d$ is the doubling dimension \citep{dasgupta2008random}, as defined below. 
\begin{definition}[Doubling dimension]
    The doubling dimension of $\mathcal{X}_{\mathsf{data}}$ is the smallest number $s$ such that for all $x \in \mathcal{X}_{\mathsf{data}}$ and $r > 0$, $B(x,r)\cap \mathcal{X}_{\mathsf{data}}$ can be covered by $2^s$ Euclidean balls with radius $r/2$.
\end{definition}
This notion of dimension can also be incorporated into our framework, as demonstrated by the following lemma \citep[Lemma 6]{kpotufe2012tree}:
\begin{lemma}\label{prop:doubling-dimension}
    If $\mathcal{X}_{\mathsf{data}}$ is contained within a $\varepsilon_0$-enlargement of a set that has doubling dimension $k$ and if Assumption~\ref{ass:bounded} holds, then Assumption~\ref{ass:low-dim} holds.
\end{lemma}

\item{\em A general set with Minkowski dimension $k$.}
As a commonly used concept of intrinsic dimension, the Minkowski dimension is defined as follows \citep{xia2010transport}. 
\begin{definition}[Minkowski dimension]
    The Minkowski dimension of $\mathcal{X}_{\mathsf{data}}$ can be defined as
    \begin{align*}
        \lim_{\varepsilon \to 0^{+}} - \frac{\log N^{\mathsf{cover}}(\mathcal{X}_{\mathsf{data}},\|\cdot\|_2,\varepsilon)}{\log \varepsilon},
    \end{align*}
	provided that the above limit exists. 
\end{definition} 
Interestingly, our Assumption~\ref{ass:low-dim} can be regarded as a weaker version of the Minkowski dimension, as stated in the following lemma.
\begin{lemma}\label{prop:Minkowski-dimension}
    If $\mathcal{X}_{\mathsf{data}}$ is contained within an $\varepsilon_0$-enlargement of a set that has Minkowski dimension $k$, then Assumption~\ref{ass:low-dim} holds.
\end{lemma}
\begin{remark}
	In some literature \citep{falconer2023minkowski}, the Minkowski dimension is defined through the limit scale of the box-covering number (denoted by $N^{\mathsf{box}\text{-}\mathsf{cover}}$) or the packing number (denoted by $N^{\mathsf{pack}}$), instead of the ball-covering number (denoted by $N^{\mathsf{cover}}$) used here.
	We remark here that all three definitions of the Minkowski dimension are equivalent, and Lemma~\ref{prop:Minkowski-dimension} holds (with only pre-constant modified) as long as $d = O(\mathrm{poly}(k))$, which can be readily seen from the following inequalities that hold for any $\varepsilon > 0$: 
    \begin{align*}
	    N^{\mathsf{box}\text{-}\mathsf{cover}}(\mathcal{X}_{\mathsf{data}},\|\cdot\|_2,\varepsilon) &\le N^{\mathsf{cover}}(\mathcal{X}_{\mathsf{data}},\|\cdot\|_2,\varepsilon/2) \le N^{\mathsf{box}\text{-}\mathsf{cover}}(\mathcal{X}_{\mathsf{data}},\|\cdot\|_2,d^{-1/2}\varepsilon),\\
	    N^{\mathsf{pack}}(\mathcal{X}_{\mathsf{data}},\|\cdot\|_2,\varepsilon) &\le N^{\mathsf{cover}}(\mathcal{X}_{\mathsf{data}},\|\cdot\|_2,\varepsilon/2) \le N^{\mathsf{pack}}(\mathcal{X}_{\mathsf{data}},\|\cdot\|_2,\varepsilon/4).
    \end{align*}
    In sum, our results remain valid when our intrinsic dimension is replaced by the Minkowski dimension.  
\end{remark}

\item {\em Sparse vectors.}
The concept of sparsity is widely used in high-dimensional statistics to characterize low-dimensional structures.
We say a vector $x \in \R^d$ is $s$-sparse if it has at most $s$ non-zero entries.
Denoting $\mathcal{S}^{d-1} = \{x \in \R^d \mid \|x\|_2 = 1\}$ to be the unit Euclidean sphere in $\R^d$, it is known that sparse vectors on $\mathcal{S}^{d-1}$ has low metric entropy~\citep{vershynin2009role}, which implies the following lemma.

\begin{lemma}
    Suppose $s < d/2$, and choose $k = s\log(d)$ in Assumption~\ref{ass:low-dim}.
    If $\mathcal{X}_{\mathsf{data}}$ is contained within an $\varepsilon_0$-enlargement of the set of $s$-sparse vectors in $\mathcal{S}^{d-1}$, then Assumption~\ref{ass:low-dim} holds.
\end{lemma}

\item {\em Union of the above sets.} 
Our framework can also cover a union of the structures mentioned above, that is, 
Assumption~\ref{ass:low-dim} automatically holds when the data are supported on the union of $k^{O(k)}$ subsets,   each of which satisfies Assumption~\ref{ass:low-dim}.
For the ImageNet dataset, the intrinsic dimension is estimated to be $k \approx 43$ by means of the Levina-Bickel method \citep{levina2004maximum}, which shares similar flavor as the doubling dimension mentioned above.
In this case, the number of subsets allowed in our framework could be as large as $10^{70}$, which manifests the flexibility and scalability of our analysis.

\end{itemize}


\subsubsection{Discretization time points}
%
%
Recall that the discretization time points are $0 = t_0  < t_1 < \cdots < t_N <t_{N+1} = T$, with $T$ the time horizon of the forward process.   
We also define the early stopping time
%
\begin{align}
    \label{eq:defn-Delta-k-interval}
 \delta \coloneqq 
 T-t_{N}<1,
\end{align}
%
the meaning of which shall be discussed momentarily when presenting our convergence theory. 



Following prior diffusion model theory (e.g., \cite{benton2024nearly}), 
we introduce a key quantity,  $\kappa$, defined to be the smallest quantity obeying
\begin{align}
\label{eq:defn-kappa}
    t_{n + 1} - t_n
    \le \kappa \min\{1,T-t_n\} 
    \qquad 
    \text{for all }0\leq n \leq N. 
\end{align}
To get a sense of the size of $\kappa$, it is helpful to bear in mind the following schedule that has been adopted in recent theoretical analysis \citep{chen2023improved,benton2024nearly,li2024towards}:
\begin{align}
t_{n} & =\begin{cases}
\frac{2(T-1)n}{N},\qquad & \text{ if }0\leq n\leq N/2,\\
T-\delta^{2(n-N/2)/N}, & ~\text{if }N/2<n\leq N, 
\end{cases}\label{eq:constant-exponential-stepsize}
\end{align}
where we assume  $N$ to be even for simplicity. 
In words, this schedule of time points can be divided into two phases: the first phase consists of the first $N/2$ time instances that are linearly spaced within $[0,T-1]$, 
whereas the second phase consists of the remaining $N/2$ time points that whose distance to $T$ decays exponentially from 1 to $\delta$. 
As shown in  \citet[Corollary~1]{benton2024nearly}, this choice \eqref{eq:constant-exponential-stepsize} satisfies
\begin{equation}
\kappa\asymp\frac{T+\log(1/\delta)}{N}.\label{eq:size-kappa}
\end{equation}
For this reason, we can often interpret $\kappa$ as the average time interval, or simply  $\widetilde{O}(1/N)$ as $T$ is typically taken to be $\widetilde{O}(1)$.  
In some other diffusion model literature, a uniform discretization schedule with rescaled forward and backward process is considered. 
Our analysis techniques can be easily adapted to this setting, and the following results, like Corollary~\ref{cor:main}, remain valid.




\subsubsection{Score estimation errors}
Given that typically only imperfect score estimates are available, 
we introduce another assumption concerning the  mean squared score estimation errors.

%
%
\begin{assumption}\label{ass:score-error}
Assume the weighted sum of the mean squared score estimation errors is finite, i.e.,  
\begin{align*}
    \sum_{n=0}^{N-1}
    (t_{n + 1} - t_n)\,
    \E\left[\big\|s_{T-t_n}(X_{T-t_n}) - \widehat{s}_{T-t_n} (X_{T-t_n})\big\|_2^2\right]  \leq \varepsilon_{\mathsf{score}}^2 < \infty,
\end{align*}
where the expectation is w.r.t.~the distribution of the forward process (cf.~\eqref{eq:forward-marginal}). 
\end{assumption}
Treating the score learning phase as a black box, 
Assumption~\ref{ass:score-error} 
singles out a single quantity,  $\varepsilon_{\mathsf{score}}\geq 0$, that captures the accuracy of the pre-trained score functions, 
which is commonly used in recent development of diffusion model theory \citep{chen2023improved,li2023towards,li2024sharp,benton2024nearly,dou2024optimal,li2024d,li2024accelerating,li2024adapting,potaptchik2024linear,li2024improved,azangulov2024convergence,liang2024non}. 
Notably, assuming a weighted average of the mean squared score errors is clearly less stringent than requiring the mean squared score estimation errors to hold uniformly over all time discretization points \citep{chen2022sampling,lee2022convergence,chen2023probability,wu2024stochastic}.



\subsection{Convergence guarantees for DDPM}
\label{sec:convergence-DDPM}
Armed with the aforementioned key assumptions, 
we are positioned to present our non-asymptotic convergence theory for the DDPM sampler~\eqref{eq:DDPM-update}. 
Here and throughout, we adopt the following convenient notation: 
\begin{itemize}
    \item $q_t$: the distribution of $X_t$ of the forward process \eqref{eq:forward-marginal};
    \item $p_t$: the distribution of $Y_t$ of the SDE specified by \eqref{eq:ddpm-mu}; 
    \item $p_{\mathsf{out}}$: the distribution of $Y_{t_N}$ (or equivalently, the distribution of the DDPM output after $N$ steps).  
\end{itemize}
%
%
\noindent 
Our convergence theory is stated below, 
whose proof is deferred to Section~\ref{sec:pf-thm-main}. 
\begin{theorem}\label{thm:main}
    Suppose that Assumptions~\ref{ass:low-dim}, \ref{ass:bounded} and \ref{ass:score-error} hold. Assume that $T > 1$, $\kappa \le 0.9$, $1/\mathrm{poly}(k) \lesssim \delta<1$ and $d = O\big(\mathrm{poly}(k)\big)$.
    Then the DDPM sampler~\eqref{eq:DDPM-update} 
    satisfies
    \begin{align}
        \mathsf{KL}(q_\delta \m p_{\mathsf{out}}) \lesssim \kappa \min\left\{\E\big[\|X_0\|_2^2\big],T k\log k\right\} + \kappa^2Nk\log k + \varepsilon_{\mathsf{score}}^2 + \big(d + \E\big[\|X_0\|_2^2\big]\big) e^{-2T}.
        \label{eq:KL-main-result}
    \end{align}
    %
\end{theorem}

The KL divergence upper bound in Theorem~\ref{thm:main} can be divided into three parts. The first two terms on the right-hand side of \eqref{eq:KL-main-result} capture the time discretization error of the DDPM dynamics (namely, the error incurred when discretizing \eqref{eq:ddpm-mu-oracle} as \eqref{eq:ddpm-mu}), 
the third term on the right-hand side of \eqref{eq:KL-main-result}  reflects the effect of the score estimation error, 
whereas the fourth term corresponds to the initialization error (namely, the error induced by drawing $Y_0$ from $\mathcal{N}(0,I_d)$ rather than the distribution of $X_T$). 
Note that we are comparing the DDPM output with $q_\delta$ (i.e., the distribution of a slightly noisy version of the data) rather than $\Pdata=q_0$, as adding a little Gaussian noise helps avoid some highly irregular cases. For this reason, $\delta>0$ is commonly referred to as the early stopping time in prior theory.  Note that $\delta$ can be taken to be very small, ensuring that $q_\delta$ is exceedingly close to $\Pdata$.  
The assumption $\kappa \le 0.9$ can be replaced by a more general condition $\kappa \le 1-\varepsilon_\kappa$ for any small constant $\varepsilon_\kappa>0$, which still suffices for the analysis.

Next, we instantiate Theorem~\ref{thm:main} to the specific discretization schedule described in \eqref{eq:constant-exponential-stepsize}, in order to obtain a clearer picture about how DDPM adapts to unknown low-dimensional structure. 
\begin{corollary}\label{cor:main}
Suppose that Assumptions~\ref{ass:low-dim}, \ref{ass:bounded} and \ref{ass:score-error} hold. 
    Choose $T = c_T\log(d/\varepsilon)$ for some large enough universal constant $c_T>0$. 
    Then the DDPM sampler~\eqref{eq:DDPM-update} with the discretization schedule \eqref{eq:constant-exponential-stepsize} achieves
     \begin{align*}
        \mathsf{KL}(q_\delta \m p_{\mathsf{out}}) \lesssim \varepsilon^2 + \varepsilon_{\mathsf{score}}^2,
    \end{align*}  
    provided that $1/\mathrm{poly}(k) \lesssim \delta<1$, $d = O\big(\mathrm{poly}(k)\big)$, and
    \begin{align}
    N \geq \frac{c_N k \log^3 (k/\varepsilon)}{\varepsilon^2}
    \end{align}
    for some large enough universal constant $c_N>0$. 
    %
%
\end{corollary}
To interpret this result, 
consider, for simplicity, the case with perfect score estimation (i.e., $\varepsilon_{\mathsf{score}}=0$) and an exceedingly small $\delta$.   
Then Corollary~\ref{cor:main} implies that: for fairly broad scenarios,  
it only takes the DDPM sampler 
\[
N = \widetilde{\Theta}(k/\varepsilon^2)~~\text{steps}
\]
to approximate the distribution of $X_{\delta}$ (i.e., the target data distribution $\Pdata$ convoluted with exceedingly small Gaussian noise) to within $\varepsilon^2$ KL divergence. 
Encouragingly, this iteration complexity is linear in the intrinsic dimension $k$ (modulo some logarithmic factor), which reduces to the state-of-the-art theory when $k=d$  (i.e., the case without low-dimensional structure) \citep{benton2024nearly}. 
It is noteworthy that all this is achieved without any prior knowledge about the potential low-dimensional structure in $\Pdata$, thereby unraveling the remarkable capability of DDPM to adapt to unknown low dimensionality. It is also worth noting that this linear dependency in $k$ holds without any additional burn-in requirement; that is, its validity is guaranteed without requiring $N$ to first exceed, say, some polynomial function in $k$ or $d$.  


To understand the optimality of DDPM in adapting to low dimensionality, 
it is helpful to first revisit the lower bounds established by \citet[Theorem 7]{chen2022sampling} and \citet[Appendix~H]{benton2024nearly} for the full-dimensional case $k=d$; namely, there exists a simple $d$-dimensional distribution $\Pdata$ for which the KL error is lower bounded by some function scaling linearly in $d$. These arguments can be straightforwardly extended to the low-dimensional case, 
revealing the existence of a distribution $\Pdata$ with intrinsic dimension $k$, that requires an iteration complexity scaling at least linearly with $k$, regardless of the sampler in use.

\paragraph{Comparison with the concurrent work~\citet{potaptchik2024linear}.} 
Assuming the data lies on a low-dimensional manifold, 
this concurrent work, posted two weeks before ours, established an iteration complexity that scales linearly in the manifold dimension. 
As alluded to previously,  our results and theirs share clearly similarities in both the aims and the results. Here, let us also point out several differences here. In \cite{potaptchik2024linear}, the data is assumed to be strictly supported on a constant-scale manifold with diameter less than $1$ (see Assumptions~A and B therein), and the density function (w.r.t.~to the standard volume measure) needs to be both upper and lower bounded uniformly (see Assumption~C therein). In comparison, we consider a more general setting where the data support is only required to stay close to a low-dimensional structure whose size can grow polynomially in $k$. Also, no boundedness requirement is imposed on the density function in our theory, thereby accommodating a much broader scenario. 
Notably, the theory in \cite{potaptchik2024linear} only applies to data with $\mathbb{E}\big[\|X_0\|_2^2\big]=O(1)$ second moment (as a consequence of their assumption that the diameter of the manifold is $O(1)$), while ours is able to accommodating polynomia lly large second moment $\mathbb{E}\big[\|X_0\|_2^2\big]=\mathrm{poly}(k)$. 
%
Our proof invokes the covering-style argument to accommodate broader low-dimensional structure, which differs from the techniques in \citet{potaptchik2024linear}.


\section{Main analysis (Proof of Theorem~\ref{thm:main})}
\label{sec:pf-thm-main}

%

We now present the proof of main convergence theory in Theorem~\ref{thm:main}. 
To control the KL divergence between the DDPM-type dynamics and our target distribution, a common approach is to connect the KL error with the $\ell_2$ discretization error by means of Girsanov's theorem 
 (e.g., \citet{chen2022sampling,chen2023score}).
We shall follow this approach in this proof, while in the meantime taking into account the intrinsic low-dimensional structure of the target distribution. In particular, we shall invoke the covering argument to treat the low-dimensional structure  approximately as the union of linear subspaces. 
Compared with previous work on low-dimensional adaptation \citep{azangulov2024convergence,li2024adapting}, instead of focusing on the influence of the intrinsic dimension evolution of the score function, we improve the $k$-dependence by analyzing the posterior mean function, i.e., $\mu_t(\cdot)$. As it turns out, the evolution of the posterior mean function can be well calibrated along the forward process with high probability under this framework, which plays a key role in enabling our non-asymptotic convergence analysis.

Our main proof is composed of three steps, with the first two steps being fairly standard and the third step capturing the main challenges in the analysis of the posterior covariance matrix $\Sigma_t(x)$ (which is directly linked to the posterior mean function and discretization error). 
In particular, Lemma~\ref{lemma:pos-var} reveals that the growth of the expected trace norm of $\Sigma_t(x)$ depends only on the intrinsic dimension $k$ as opposed to the ambient dimension $d$, thereby leading to a refined control of the discretization error.
The analysis is supported by a sequence of auxiliary lemmas whose proofs are deferred to Appendices~\ref{sec:pf-main-lems} and \ref{sec:pf-supp-lems}. 
Throughout the proof, we adopt the following convenient notation: 
\begin{itemize}
    \item $Q$: the measure of the reverse  process \eqref{eq:backward};
    \item $\pi_d$: the  standard Gaussian distribution in $\R^d$; 

    \item  $P^{(T)}$: the measure of the DDPM dynamics \eqref{eq:ddpm-SDE-adaptive} when initialized with $Y_0 \sim X_T$; 
    
    \item  $P^{(\infty)}$: the measure of DDPM dynamics \eqref{eq:ddpm-SDE-adaptive} when initialized with $Y_0 \sim \pi_d \overset{\mathrm{d}}{=} X_\infty$, where $X_\infty$ is the limit of the forward process \eqref{eq:forward} as $t \to \infty$.

\end{itemize}

\paragraph{Step 1: Bounding the KL divergence between two path measures.}
In this step, we list several lemmas that help connect the KL error of interest with the mean squared difference between the true and estimated posterior means. These results, which have been used in past work like \citet{chen2022sampling,chen2023improved,benton2024nearly,azangulov2024convergence,wu2024stochastic,potaptchik2024linear},  are now standard for analyzing DDPM-type samplers.    

To begin with, we upper bound  the KL error of interest, $\KL(q_\delta \m p_{\mathsf{output}})$, 
by means of the KL divergence between two path measures $Q$ and $P^{(\infty)}$.  
This is accomplished through the following standard lemma; 
its proof is fairly standard and we refer the interested reader to \citet[Section 3.3]{benton2024nearly}. 
%
%
\begin{lemma}\label{lemma:error-decompose}
    If $\KL(Q \m P^{(T)}) < \infty$, then we have 
    \begin{align*}
        \KL(q_\delta \m p_{\mathsf{out}}) = \KL(q_\delta \m p_{t_N}) \le \KL(Q \m P^{(\infty)}) = \KL(Q \m P^{(T)}) + \KL(q_T \m \pi_d).
    \end{align*}
\end{lemma}


In view of Lemma~\ref{lemma:error-decompose}, to establish Theorem~\ref{thm:main}, it boils down to controlling $\KL(Q \m P^{(T)})$ and $\KL(q_T \m \pi_d)$ separately. 
Let us start with the last term $\KL(q_T \m \pi_d)$, which is equivalent to analyzing the convergence rate of the forward process. 
The following lemma offers a simple upper bound on this term. Note that the exponential convergence of the OU process is well-known; 
we refer the reader to 
\citet[Lemma~9]{chen2023improved} or  \citet[Appendix~G]{benton2024nearly} for the proof this lemma.
%
\begin{lemma}\label{lemma:OU}
    Suppose that $\E[\|X_0\|^2] < \infty$.  Then the forward OU process \eqref{eq:forward} converges exponentially fast to its stationary distribution in KL divergence, i.e., for any $T > 1$ one has
    \begin{align}
        \KL(q_T \m \pi_d) \lesssim \left(d + \E[\|X_0\|_2^2]\right)e^{-2T}.\notag
    \end{align}
\end{lemma}
%
%
As a result of Lemmas~\ref{lemma:error-decompose} and \ref{lemma:OU}, it remains to control the term $\KL(Q \m P^{(T)})$. 
The following lemma links $\KL(Q \m P^{(T)})$ with the mean squared difference between the drift terms of SDEs~\eqref{eq:ddpm-mu} and \eqref{eq:ddpm-mu-oracle}, 
a standard result that can be readily accomplished via the Girsanov theorem \citep{le2016brownian}. 
\begin{lemma}\label{lemma:error-dis-score}
    For any time discretization schedule with $t_0 = 0$ and $t_N = T-\delta$,
    it holds that
    \begin{align}
        \KL(Q \m P^{(T)}) \le \sum_{n=0}^{N-1} \int_{t_n}^{t_{n+1}} \frac{1-\sigma^2_{T-t}}{\sigma_{T-t}^4} \E_Q\left[\big\|\mu_{T-t}(Y_t) - \widehat{\mu}_{T-t_n}(Y_{t_n})\big\|_2^2\right] \d t < \infty.
        \label{eq:error-dis-score}
    \end{align}
\end{lemma}
%
The proof follows from an application of Girsanov's Theorem to SDEs~\eqref{eq:ddpm-mu} and \eqref{eq:ddpm-mu-oracle}, with the aid of a coupling argument. 
Given that this standard trick and its derivation has appeared in many recent papers (e.g., \citet[Section 5.2]{chen2022sampling} or \citet[Appendix F]{benton2024nearly}), we omit it for brevity. 
%

In what follows, we 
shall often suppress the subscript $Q$ in the  expectation operator as long as it is clear from the context. 
To proceed, note that the middle expression in inequality \eqref{eq:error-dis-score} can be further bounded by the Cauchy-Schwarz inequality as 
\begin{align}
    &\sum_{n=0}^{N-1} \int_{t_n}^{t_{n+1}} \frac{1-\sigma^2_{T-t}}{\sigma_{T-t}^4} \E\left[\big\|\mu_{T-t}(Y_t) - \widehat{\mu}_{T-t_n}(Y_{t_n})\big\|_2^2\right] \d t \notag\\
    &\quad \lesssim \underset{\eqqcolon\, T_1}{\underbrace{\sum_{n=0}^{N-1}\int_{t_n}^{t_{n+1}} \left(\frac{\eta_{T-t}}{\eta_{T-t_n}}\right)^2 \E\left[\big\|s_{T-t_n}(Y_{t_n}) - \widehat{s}_{T-t_n} (Y_{t_n})\big\|_2^2\right] \d t }} \notag\\
    &\qquad\qquad + \underset{\eqqcolon\, T_2}{\underbrace{\sum_{n=0}^{N-1} \int_{t_n}^{t_{n+1}} \frac{1-\sigma^2_{T-t}}{\sigma_{T-t}^4} \E\left[\big\|\mu_{T-t}(Y_t) - \mu_{T - t_n}(Y_{t_n})\big\|_2^2 \right] \d t}}.
    \label{eq:decompose-MSE-T12}
\end{align}
where we have used the definition \eqref{eq:defn-eta-t} of $\eta_t$ and the following decomposition (according to \eqref{eq:ep-mu-s}) 
\begin{align*}
\mu_{T-t}(Y_{t})-\widehat{\mu}_{T-t_{n}}(Y_{t_{n}}) & =\big(\mu_{T-t_{n}}(Y_{t_{n}})-\widehat{\mu}_{T-t_{n}}(Y_{t_{n}})\big)+\big(\mu_{T-t}(Y_{t})-\mu_{T-t_{n}}(Y_{t_{n}})\big)\\
 & =\frac{\sigma_{t_{n}}^{2}}{\sqrt{1-\sigma_{t_{n}}^{2}}}\big(s_{T-t_{n}}(Y_{t_{n}})-\widehat{s}_{T-t_{n}}(Y_{t_{n}})\big)+\big(\mu_{T-t}(Y_{t})-\mu_{T-t_{n}}(Y_{t_{n}})\big).
\end{align*}
In the above decomposition \eqref{eq:decompose-MSE-T12}, $T_1$ reflects the influence of the score estimation error, whereas $T_2$ captures the effect of the time  discretization error.
In the sequel, we shall upper bound $T_1$ and $T_2$ separately, and combine them to reach our final bound on $\KL(Q \m P^{(T)})$.

\paragraph{Step 2: Bounding the effect of score estimation errors.}

To cope with the term $T_1$, we find it helpful to first establish the following result on the coefficients $\{\eta_t\}$, whose proof is provided in Appendix~\ref{sec:proof:lemma:score-weight}. 
The arguments of this step are also fairly standard, akin to the ones used in, e.g.,  \cite{azangulov2024convergence,potaptchik2024linear}. 
\begin{lemma}\label{lemma:score-weight}
There exists some universal constant $C_3>0$, independent of $n$, such that
    \begin{align}
        \eta_{T-t}/\eta_{T-t_n} \le C_3 \qquad \text{for any } t \in [t_n,t_{n+1}).\notag
    \end{align}
\end{lemma}
With Lemma~\ref{lemma:score-weight} in mind, we know the weight $(\eta_{T-t}/\eta_{T-t_n})^2$ can be uniformly bounded by $O(1)$.
Thus, 
\begin{align}
    T_1 &= \sum_{n=0}^{N-1}\int_{t_n}^{t_{n+1}} \left(\frac{\eta_{T-t}}{\eta_{T-t_n}}\right)^2 \E\left[\big\|s_{T-t_n}(Y_{t_n}) - \widehat{s}_{T-t_n} (Y_{t_n})\big\|_2^2\right] \d t\notag\\
    & \lesssim \sum_{n=0}^{N-1}\int_{t_n}^{t_{n+1}} \E\left[\big\|s_{T-t_n}(Y_{t_n}) - \widehat{s}_{T-t_n} (Y_{t_n})\big\|_2^2\right] \d t\notag\\
    & = \sum_{n=0}^{N-1}(t_{n+1} - t_n)\, \E\left[\big\|s_{T-t_n}(Y_{t_n}) - \widehat{s}_{T-t_n} (Y_{t_n})\big\|_2^2\right] \d t \leq \varepsilon_{\mathsf{score}}^2,\label{eq:t1}
\end{align}
where the last equality arises from Assumption~\ref{ass:score-error}.

\paragraph{Step 3: Bounding the effect of discretization errors.}

We now turn to the term $T_2$ in \eqref{eq:decompose-MSE-T12}, which constitutes the main challenge of establishing linear-$k$ dependency.   
Here, $(Y_t)_{t \in [0,T]}$ stands for the solution to the backward SDE~\eqref{eq:backward} when initialized at $Y_0\sim X_T$, 
which satisfies $X_t \overset{\mathrm{d}}{=}Y_{T-t}$ (cf.~\eqref{eq:equiv-Y-X-reverse}).   
Towards this end, we intend to first analyze each term in the integral. 
Define 
\begin{align}
    D_{s,t} \coloneqq \frac{1-\sigma^2_{T-t}}{\sigma_{T-t}^4} \E\left[\big\|\mu_{T-t}(Y_t) - \mu_{T - s}(Y_{s})\big\|_2^2\right],
    \qquad t\geq s.
    \label{eq:defn-Dst}
\end{align}
%
In the following, we shall demonstrate that $D_{s,t}$ is, in some sense, adaptive to the (unknown) low-dimension structure of the target distribution $\Pdata$. 
More specifically, we would like to prove that for a given $s$,  $D_{s,t}$ changes slowly in $t$ when $\Pdata$ has a low-dimensional structure, i.e., $\partial_t D_{s,t}$ scales with $k$ instead of $d.$

Towards this end, we proceed by characterizing the stochastic evolution of $\mu_{T-t}(Y_t)$. 
We begin with the 
following lemma, which leverages It\^{o}'s isometry to derive an alternative expression for $D_{s,t}$ that lends it well for subsequent analysis. 
The proof is provided in Appendix~\ref{sec:proof:lemma:dst}. 
\begin{lemma}\label{lemma:dst}
    Recall that $\Sigma_t(x) = \mathrm{Cov}[X_0|X_t = x]$. 
    The quantity $D_{s,t}$ defined in \eqref{eq:defn-Dst} satisfies 
    \begin{align}
        D_{s,t} = \frac{1-\sigma^2_{T-t}}{\sigma_{T-t}^4} \Big(\E\big[\tr(\Sigma_{T-s}(Y_s)\big] - \E\big[\tr(\Sigma_{T-t}(Y_t)\big]\Big).
        \label{eq:Dst-alt-expression}
    \end{align}
\end{lemma}
\begin{remark}
Our approach shares similar spirits as \citet[Lemma 4]{benton2024nearly}. Similar characterizations were obtained in Section~4 of the concurrent work \citet{potaptchik2024linear} by analyzing the martingale structure underlying the quantity of interest. 
\end{remark}
Given that $D_{s,t} \ge 0$ for $t\geq s$ (which follows directly from its definition \eqref{eq:defn-Dst}),  
 Lemma~\ref{lemma:dst} implies that %
 \begin{align}
    \E\left[\tr\big(\Sigma_{T-t}(Y_t)\big)\right] 
    \text{ is non-increasing in }t.
    \label{eq:non-increasing-Sigma}
    \end{align}
We are now ready to present a key lemma that can assist in upper bounding $D_{s,t}$, 
which controls the size of the posterior covariance. The proof can be found in Appendix~\ref{sec:proof:lemma:pos-var}.
%
%
\begin{lemma}\label{lemma:pos-var}
    For all $t \in [0,T)$, there exists some universal constant $C_4>0$ such that
    \begin{align*}
        \E\left[\tr\big(\Sigma_{T-t}(Y_t)\big)\right] \le C_4 \min\left\{\E[\|X_0\|_2^2], \frac{\sigma^2_{T-t}}{1-\sigma^2_{T-t}}k\log k\right\}.
    \end{align*}
\end{lemma}
\begin{remark}
Similar strategies were introduced in the previous work \citet[Appendix E]{benton2024nearly},  without accounting for the low-dimensional structure.  
It is worth noting that Lemma~9 in the concurrent work \cite{potaptchik2024linear} shares similarity with the above result too. Note, however, that our lemma is established under much weaker assumptions on the low-dimensional structure as well as a much large range of the covariance of $X_0$, as discussed towards the end of Section~\ref{sec:convergence-DDPM}.
\end{remark}
With the above derivations in mind, 
we are in a position to control $T_2$. 
First, let us rewrite $T_2$ using Lemma~\ref{lemma:dst} as
\begin{align}
    T_2 = \sum_{n=0}^{N-1} \int_{t_n}^{t_{n+1}} D_{t_n,t} \d t = \sum_{n=0}^{N-1} \int_{t_n}^{t_{n+1}} \frac{1-\sigma^2_{T-t}}{\sigma_{T-t}^4} \Big(\E\left[\tr\big(\Sigma_{T-t_n}(Y_{t_n})\big)\right] - \E\left[\tr\big(\Sigma_{T-t}(Y_t)\big)\right]\Big)\d t.\label{eq:dst-integral}
\end{align}
In the following analysis, we decompose the time horizon $[0,T-\delta]$ into $[0,T-1]$ and $[T-1,T-\delta]$, and control the right-hand side of equation~\eqref{eq:dst-integral} on these two time intervals separately. Without loss of generality, suppose that the $M$-th discretization time point obeys $t_M = T-1$. 
%

\paragraph{(i) The case when $t \in [0,T-1]$.} By the monotonicity of $\sigma_{T-t}$, we have
\begin{align*} 
    \frac{1-\sigma^2_{T-t}}{\sigma_{T-t}^4} \le \frac{1-\sigma^2_{T-t}}{\sigma_1^4} = C_5(1-\sigma^2_{T-t}) < C_5,
\end{align*}
where $C_5 \coloneq 1/\sigma_1^{4}$ is a numerical constant.
Direct calculation gives
\begin{align}
    \sum_{n=0}^{M-1} \int_{t_n}^{t_{n+1}} D_{t_n,t} \d t &\le \sum_{n=0}^{N-1} \int_{t_n}^{t_{n+1}} C_5 \Big(\E\left[\tr\big(\Sigma_{T-t_n}(Y_{t_n})\big)\right] - \E\left[\tr\big(\Sigma_{T-t}(Y_t)\big)\right]\Big)\d t \label{eq:sum-0m}\\
    &\le C_5 \kappa \E\left[\tr\big(\Sigma_{T}(Y_0)\big)\right] \lesssim  \kappa \E[\|X_0\|_2^2],\label{eq:error-con-1}
\end{align}
where the last line applies  Lemma~\ref{lemma:pos-var} to obtain an upper bound on $\E\left[\tr\big(\Sigma_{T}(Y_0)\big)\right]$.
Another way to bound the left-hand side 
of inequality~\eqref{eq:sum-0m} is as follows:
\begin{align}
    \sum_{n=0}^{M-1} \int_{t_n}^{t_{n+1}} D_{t_n,t} \d t & \overset{\mathrm{(a)}}{\le} C_5 \sum_{n=0}^{M-1} \int_{t_n}^{t_{n+1}} (1-\sigma^2_{T-t_{n+1}}) \Big(\E\left[\tr\big(\Sigma_{T-t_n}(Y_{t_n})\big)\right] - \E\left[\tr\big(\Sigma_{T-t}(Y_t)\big)\right]\Big)\d t \notag\\
    &\overset{\mathrm{(b)}}{\lesssim} (1-\sigma^2_{T-\kappa})\kappa\E\left[\tr\big(\Sigma_{T}(Y_0)\big)\right]  + \sum_{n=1}^{M-1} (\sigma_{T-t_{n}}^2 - \sigma_{T-t_{n+1}}^2)\kappa \E\left[\tr\big(\Sigma_{T-t_n}(Y_{t_n})\big)\right]\notag\\
    &\overset{\mathrm{(c)}}{\lesssim}  \frac{(1-\sigma^2_{T-\kappa})\sigma^2_{T}}{1-\sigma^2_{T}}\kappa k\log k \notag\\
    &\qquad \qquad \qquad + \sum_{n=1}^{M-1} (1-\sigma_{T-{t_{n+1}}}^2)\big(1-e^{-2(t_{n+1}-t_{n})}\big)\cdot \frac{\sigma^2_{T-t_n}}{1-\sigma^2_{T-t_n}}\kappa k\log k\notag\\
    &\overset{\mathrm{(d)}}{\lesssim} e^{2\kappa}\kappa k\log k + \sum_{n=1}^{M-1} (t_{n+1}-t_{n})\kappa k\log k\notag\\
    &\overset{\mathrm{(e)}}{\lesssim} \kappa k\log k + T\kappa  k\log k,\label{eq:error-con-2}
\end{align}
where (a) follows from Lemma~\ref{lemma:dst} and the fact that $\sigma_{T-t}\leq 1$, 
(b) makes use of the fact that $t_{n+1}-t_n\leq \kappa$, 
(c) invokes Lemma~\ref{lemma:pos-var}, 
(d) relies on the facts that $\frac{1-\sigma_{T-t_{n+1}}^{2}}{1-\sigma_{T-t_{n}}^{2}}=e^{2(t_{n+1}-t_{n})}\lesssim1$, $\sigma_{T-t_n}\leq 1$ and $1-e^{-2(t_{n+1}-t_{n})}\lesssim t_{n+1}-t_n$,  
and (e) holds since $\kappa\leq 0.9$ and $\sum_{n=1}^{M-1}(t_{n+1}-t_n)\leq T$. 
Combining inequalities~\eqref{eq:error-con-1} and \eqref{eq:error-con-2} gives
\begin{align*}
    \sum_{n=0}^{M-1} \int_{t_n}^{t_{n+1}} D_{t_n,t} \d t \lesssim \kappa \min\left\{\E[\|X_0\|_2^2], Tk\log k\right\}.
\end{align*}

\paragraph{(ii) The case when $t \in [T-1,T-\delta]$.} In this case, it is easily seen that
\begin{align*}
    \frac{1-\sigma^2_{T-t}}{\sigma_{T-t}^4} = \frac{e^{-2(T-t)}}{(1-e^{-2(T-t)})^2} 
    \lesssim \frac{1}{(T-t)^2}.
\end{align*}
Taking this together with  expression~\eqref{eq:dst-integral} yields
\begin{align}
    \sum_{n=M}^{N-1} \int_{t_n}^{t_{n+1}} D_{t_n,t} \d t &\lesssim \sum_{n=M}^{N-1} \int_{t_n}^{t_{n+1}} \frac{1}{(T-t)^2} \Big(\E\left[\tr\big(\Sigma_{T-t_n}(Y_{t_n})\big)\right] - \E\left[\tr\big(\Sigma_{T-t}(Y_t)\big)\right]\Big)\d t\notag\\
    &\overset{\mathrm{(a)}}{\le} \sum_{n=M}^{N-1} \int_{t_n}^{t_{n+1}} \frac{1}{(T-t_{n+1})^2} \Big(\E\left[\tr\big(\Sigma_{T-t_n}(Y_{t_n})\big)\right] - \E\left[\tr\big(\Sigma_{T-t_{n+1}}(Y_{t_{n+1}})\big)\right]\Big)\d t\notag\\
    &\overset{\mathrm{(b)}}{\le} \sum_{n=M}^{N-1} \frac{\kappa}{1-\kappa} \cdot \frac{1}{(T - t_{n+1})}\Big(\E\left[\tr\big(\Sigma_{T-t_n}(Y_{t_n})\big)\right] - \E\left[\tr\big(\Sigma_{T-t_{n+1}}(Y_{t_{n+1}})\big)\right]\Big)\notag\\
    &\overset{\mathrm{(c)}}{\le} \frac{\kappa}{(1-\kappa)^2}\E\left[\tr\big(\Sigma_1(Y_{T-1}\big)\right] + \frac{\kappa^2}{(1-\kappa)^2}\sum_{n=M+1}^{N-1} \frac{1}{T-t_n} \E\left[\tr\big(\Sigma_{T-t_n}(Y_{t_n})\big)\right]\notag\\
    &\overset{\mathrm{(d)}}{\lesssim} \frac{\kappa}{(1-\kappa)^2}\E\left[\tr\big(\Sigma_1(Y_{T-1}\big)\right] + \frac{\kappa^2}{(1-\kappa)^2}\sum_{n=M+1}^{N-1} \frac{1}{T-t_n} \cdot  k \log k \frac{\sigma_{T-t}^2}{1-\sigma_{T-t}^2}\notag\\
    &\overset{\mathrm{(e)}}{\lesssim} \kappa \min\big\{\E[\|X_0\|_2^2],k\log k\big\} + \kappa^2 \sum_{n=M+1}^{N-1} k\log k\notag\\
    &\le \kappa \min\big\{\E[\|X_0\|_2^2],k\log k\big\} + \kappa^2 Nk \log k.  \label{eq:sum-mn}
\end{align}
Here, (a) is due to the non-increasing property in \eqref{eq:non-increasing-Sigma}; 
(b) and (c) invoke the following fact (see \eqref{eq:defn-kappa}): %
\[
t_{n+1}-t_{n}\leq\kappa(T-t_{n})=\kappa(T-t_{n+1})+\kappa(t_{n+1}-t_{n})\quad\Longrightarrow\quad t_{n+1}-t_{n}\leq\frac{\kappa}{1-\kappa}(T-t_{n+1});
\] 
(c) arises from Lemma~\ref{lemma:pos-var};  
and (d) holds since $\kappa\leq 0.9$. 

\paragraph{(iii) Combining these two cases. }
With inequalities~\eqref{eq:sum-0m} and \eqref{eq:sum-mn} in place, we reach
\begin{align}
    T_2 &= \sum_{n=0}^{M-1} \int_{t_n}^{t_{n+1}} D_{t_n,t} \d t + \sum_{n=M}^{N-1} \int_{t_n}^{t_{n+1}} D_{t_n,t} \d t
    \lesssim \kappa \min\left\{\E[\|X_0\|_2^2],T k\log k\right\} + \kappa^2Nk\log k. \label{eq:t2}
\end{align}

\paragraph{Step 4: Putting all pieces together. }
Combining inequalities~\eqref{eq:t1} and \eqref{eq:t2} with Lemma~\ref{lemma:error-dis-score} gives
\begin{align}
    \KL(Q \m P^T) \le T_1 + T_2 \lesssim \kappa \min\left\{\E[\|X_0\|^2],T k\log k\right\} + \kappa^2Nk\log k + \varepsilon_{\mathsf{score}}^2.\label{eq:kl-q-pt-bound}
\end{align}
Finally, taking this result together with Lemma~\ref{lemma:OU}  and Lemma~\ref{lemma:error-decompose} completes the proof of Theorem~\ref{thm:main}.



\section{Discussion}
\label{sec:discussion}

This paper has made progress towards demystifying the unreasonable effectiveness of then DDPM when the target distribution has a low-dimensional structure. We have sharpened the past result in \cite{li2024adapting}  by proving that the iteration complexity scales linearly with the intrinsic dimension $k$ (note that linear dependency in $k$ was also established in the concurrent work \cite{potaptchik2024linear}).  
%
%
In some sense, our results reveal that the DDPM sampler is optimally adaptive to unknown low dimensionlity, provided that a suitable discretization schedule is adopted. 
Our theory leverages the intimate connection between the DDPM sampler and an adaptively discretized SDE, where the non-linear component of the drift time of this SDE is inherently low-dimensional.  

We conclude this paper by pointing out several directions worthy of future investigation.
To begin with, 
it remains unclear how the DDPM sampler adapts to unknown low dimensionality when other metrics (e.g., the total-variation distance or the Wasserstein distance) are used to measure distributional discrepancy.
For example, the FID score --- commonly adopted in practice --- is not fully captured by the KL divergence, but instead should be interpreted as a mixture of other metrics. However, 
existing convergence analysis under other metrics typically requires much more stringent assumptions like smoothness on the (noisy) score function and/or log-concavity on the target distribution, which remains much less developed. 
Another interesting direction is to examine the adaptivity of other mainstream samplers (e.g., the probability flow ODE \citep{song2020score,chen2023restoration,li2024sharp,gao2024convergence,huang2024convergence,liang2025low,li2025faster}) in the face of unknown low-dimensional structure. 
Although the ODE-based samplers often achieve better iteration complexity than the SDE counterpart in practice, the convergence rate of ODE-based sampler in KL divergence indeed cannot surpass the linear rate. 
It would be great to see whether ODE samplers can achieve more favorable convergence guarantees under certain practically relevant metrics like the Wasserstein distance. 
Moreover, this paper treats the training procedure as a black box and directly assumes a mean squared score estimation error bound. 
For distributions with low intrinsic dimensions, a desirable score learning algorithm would adapt to the low-dimensional structure too, which has been the focus of a recent line of work  \citep{kadkhodaie2023generalization,tang2024adaptivity,chen2024exploring,azangulov2024convergence,chen2023score,wang2024diffusion,wang2024evaluating}.  It would thus be interesting to develop an end-to-end theory to accommodate both the pre-training phase and the sampling process.  
Finally, our analysis, which highlights the crucial role of proper parameterization and discretization, might shed light on how to design the coefficients for other algorithms (e.g., higher-order SDE and ODE solvers) in order to accommodate inherently structured data distributions.

\section*{Acknowledgements}

Y.~Wei is supported in part by the NSF grant CCF-2418156 and the CAREER award DMS-2143215. Y.~Chen is supported in part by the Alfred P.~Sloan Research Fellowship, the ONR grants N00014-22-1-2354 and N00014-25-1-2344, the NSF grants 2221009 and 2218773, the Wharton AI \& Analytics Initiative's AI Research Fund, and the Amazon Research Award.
We would like to thank Jiadong Liang and Yuchen Wu for their inspiring discussions.
We would also like to thank the authors of \citet{potaptchik2024linear,azangulov2024convergence} for discussing the connections between our work and theirs, particularly for pointing out \citet[Proposition 28]{azangulov2024convergence} that established the SDE equivalence of the DDPM.

\appendix


\section{Proof of Theorem~\ref{thm:ddpm}}
\label{sec:proof-prop:ddpm}



In what follows, we justify the equivalence of SDE~\eqref{eq:ddpm-mu} and the update rule of the DDPM.

\paragraph{Step 1: Solving SDE~\eqref{eq:ddpm-mu}.}
We proceed by directly solving SDE~\eqref{eq:ddpm-mu}. 
Let $(Y_t)_{t\in[0,T-\delta]}$ be the solution to SDE~\eqref{eq:ddpm-mu}. 
Multiply $Y_t$ by some  differentiable function $f(\cdot)$ (to be specified shortly) and apply It\^{o}'s formula \citep{oksendal2003stochastic} to yield
\begin{align}
    \d\big(f(t)Y_t\big) 
    &= \big(f'(t) + f(t)(1-2\sigma^{-2}_{T-t})\big)Y_t \d t + 2f(t)\sqrt{1-\sigma_{T-t}^2}\sigma^{-2}_{T-t} \widehat{\mu}_{T-t_n}(Y_{t_n})\d t + \sqrt{2}f(t) \d B_t, 
    \label{eq:diff-f-Yt}
\end{align}
%
To get rid of the linear term in \eqref{eq:diff-f-Yt}, 
a natural strategy is to choose $f(\cdot)$ to obey
\begin{align}
    f'(t) + f(t)\big(1-2\sigma^{-2}_{T-t}\big) = 0,
    \qquad \text{or equivalently,}
    \qquad 
f'(t)-f(t)\frac{1+e^{-2(T-t)}}{1-e^{-2(T-t)}}=0,
    \label{eq:ODE-ft-obey}
\end{align}
which forms a first-order ODE and admits a non-zero solution as follows
\begin{align}
    \frac{\d\big(\log f(t)\big)}{\d t} = \frac{1 + e^{-2(T-t)}}{1- e^{-2(T-t)}} \qquad\Longrightarrow\qquad f(t) = \frac{e^{-(T-t)}}{1-e^{-2(T-t)}}.
    \label{eq:specific-f-ODE}
\end{align}
With this choice of $f(\cdot)$ in place,  one can eliminate the linear term in \eqref{eq:diff-f-Yt} and obtain
\begin{align}
    \d\big(f(t)Y_t\big) = 2 \big(f(t)\big)^2 \widehat{\mu}_{T-t_n}(Y_{t_n})\d t + \sqrt{2} f(t) \d B_t, \notag
\end{align}
where we use the basic fact that $f(t)=\sqrt{1-\sigma_{T-t}^{2}}\sigma_{T-t}^{-2}$. 
Integrating both sides from $t_n$ to $t_{n+1}$ gives
\begin{align}
    f(t_{n+1})Y_{t_{n+1}} = f(t_n)Y_{t_n} + 2\int_{t_n}^{t_{n+1}} \big(f(t)\big)^2 \widehat{\mu}_{T-t_n}(Y_{t_n}) \d t + \sqrt{2}\int_{t_n}^{t_{n+1}} f(t) \d B_t. 
    \label{eq:sol-SDE-reparm}
\end{align}
In summary, this solves SDE~\eqref{eq:ddpm-mu} with the aid of the choice \eqref{eq:specific-f-ODE} of $f(\cdot)$.

\paragraph{Step 2: Establishing the equivalence between \eqref{eq:sol-SDE-reparm} and DDPM.}
Armed with the solution \eqref{eq:sol-SDE-reparm}, we are positioned to justify its equivalence to the DDPM update rule.  
From the quadratic variation of the Brownian motion, we can write, for each $0\leq n<N$, 
\begin{equation}
\sqrt{2}\int_{t_n}^{t_{n+1}} f(t) \d B_t 
= \left(\int_{t_n}^{t_{n+1}} 2\big(f(t)\big)^2 \d t\right)^{1/2} \widetilde{Z}_{n}\notag
\end{equation}
%
for some vector $\widetilde{Z}_{n} \sim \N(0,I_d)$, 
where $\{\widetilde{Z}_{n}\}_{n=0,\dots,N-1}$ are independent. 
This allows us to write
%
%
\begin{align}
    f(t_{n+1})Y_{t_{n+1}} 
    = f(t_n)Y_{t_n} + \left(\int_{t_n}^{t_{n+1}} 2\big(f(t)\big)^2 \d t \right) \widehat{\mu}_{T-t_n}(Y_{t_n}) + \left(\int_{t_n}^{t_{n+1}} 2\big(f(t)\big)^2 \d t\right)^{1/2} \widetilde{Z}_{n},\label{eq:ddpm-integral}
\end{align}
%
By defining 
$\gamma_n = e^{-(T-t_n)}$.
we can explicitly calculate the integral as follows
\begin{align*}
    \int_{t_n}^{t_{n+1}} 2\big(f(t)\big)^2 \d t = \int_{t_n}^{t_{n+1}} \frac{2e^{-2(T-t)}}{(1-e^{-2(T-t)})^2} \d t = \frac{\gamma_{n+1}^2 - \gamma_n^2}{(1-\gamma_n^2)(1-\gamma_{n+1}^2)}.
\end{align*}
Substitution into equation~\eqref{eq:ddpm-integral} then gives
\begin{align*}
Y_{t_{n+1}} & =\frac{1}{f(t_{n+1})}\left\{ f(t_{n})Y_{t_{n}}+\left(\int_{t_{n}}^{t_{n+1}}2\big(f(t)\big)^{2}\d t\right)\widehat{\mu}_{T-t_{n}}(Y_{t_{n}})+\left(\int_{t_{n}}^{t_{n+1}}2\big(f(t)\big)^{2}\d t\right)^{1/2}\widetilde{Z}_{n}\right\} \\
 & =\frac{1-\gamma_{n+1}^{2}}{\gamma_{n+1}}\left\{ \frac{\gamma_{n}}{1-\gamma_{n}^{2}}Y_{t_{n}}+\frac{\gamma_{n+1}^{2}-\gamma_{n}^{2}}{(1-\gamma_{n}^{2})(1-\gamma_{n+1}^{2})}\widehat{\mu}_{T-t_{n}}(Y_{t_{n}})+\left(\frac{\gamma_{n+1}^{2}-\gamma_{n}^{2}}{(1-\gamma_{n}^{2})(1-\gamma_{n+1}^{2})}\right)^{1/2}\widetilde{Z}_{n}\right\} \\
\\
 & \overset{\mathrm{(a)}}{=}\frac{1-\gamma_{n+1}^{2}}{\gamma_{n+1}}\cdot\frac{\gamma_{n}}{1-\gamma_{n}^{2}}Y_{t_{n}}+\frac{\gamma_{n+1}^{2}-\gamma_{n}^{2}}{(1-\gamma_{n}^{2})\gamma_{n}\gamma_{n+1}}\left(Y_{t_{n}}+(1-\gamma_{n}^{2})\widehat{s}_{T-t_{n}}(Y_{t_{n}})\right)\\
 & \qquad \qquad \qquad \qquad \qquad \qquad \qquad \qquad \qquad \qquad \qquad 
 + \sqrt{\frac{(\gamma_{n+1}^{2}-\gamma_{n}^{2})(1-\gamma_{n+1}^{2})}{(1-\gamma_{n}^{2})\gamma_{n+1}^{2}}}\widetilde{Z}_{n}\\
 & =\frac{\gamma_{n+1}}{\gamma_{n}}Y_{t_{n}}+\frac{\gamma_{n+1}^{2}-\gamma_{n}^{2}}{\gamma_{n}\gamma_{n+1}}\widehat{s}_{T-t_{n}}(Y_{t_{n}})+ \sqrt{\frac{(\gamma_{n+1}^{2}-\gamma_{n}^{2})(1-\gamma_{n+1}^{2})}{(1-\gamma_{n}^{2})\gamma_{n+1}^{2}}}\widetilde{Z}_{n},
\end{align*}
where (a) follows from the relationship between $\widehat{\mu}_{t}$ and $\widehat{s}_{t}$ (cf.~Eq.~\eqref{eq:ep-mu-s}).
To finish up, 
recalling that $\alpha_n = \exp(2t_{n}-2t_{n+1})= \gamma_n^2/\gamma_{n+1}^2
$ and $\overline{\alpha}_n = e^{-2(T-t_n)}=\gamma_n^2$ 
(cf.~\eqref{eq:alpha-bar-n}),  we arrive at
\begin{align*}
    Y_{t_{n+1}} = \frac{1}{\sqrt{\alpha_n}} \left(Y_{t_n} + (1-\alpha_n)\widehat{s}_{T-t_n}(Y_{t_n}) + \sqrt{\frac{(1-\alpha_n)(\alpha_n - \overline{\alpha}_n)}{1-\overline{\alpha}_n}}\widetilde{Z}_{n}\right),
\end{align*}
which coincides with the DDPM update rule \eqref{eq:DDPM-update} introduced in \cite{ho2020denoising}.

\section{Proofs of the main lemmas}
\label{sec:pf-main-lems}

\subsection{Proof of Lemma~\ref{lemma:score-weight}}
\label{sec:proof:lemma:score-weight}
For any $t \in [t_n,t_{n+1})$, the definition of $\kappa$ gives $t-t_n < \kappa$.
Therefore,
\begin{align*}
    \frac{\eta_{T-t}}{\eta_{T-t_n} }&= \frac{e^{-(T-t)}}{1-e^{-2(T-t)}} \cdot \frac{1-e^{-2(T-t_n)}}{e^{-(T-t_n)}} = e^{t-t_n}\cdot \frac{1-e^{-2(T-t_n)}}{1-e^{-2(T-t)}}
    \le e^{\kappa} \cdot \frac{1-e^{-2(T-t_n)}}{1-e^{-2(T-t_{n+1})}}. 
\end{align*}
Observing that
$
    T - t_{n+1} = T - t_n - (t_{n+1} - t_n) \ge (1-\kappa)(T-t_n)$,
we can obtain
\begin{align*}
    \frac{\eta_{T-t}}{\eta_{T-t_n} }\le e^{\kappa} \cdot \frac{1-e^{-2(T-t_n)}}{1-e^{-2(1-\kappa)(T-t_n)}} \lesssim e^{\kappa} \cdot \frac{1}{1-\kappa} 
    \lesssim 1. 
\end{align*}
%

\subsection{Proof of Lemma~\ref{lemma:dst}}
\label{sec:proof:lemma:dst}
For notational simplicity, set $R_t \coloneqq s_{T-t}(Y_t)$. We first present the SDE representation of $R_t$ as follows
%
%
\begin{align}
    \d R_t =  
    -R_t \d t + \sqrt{2} \nabla s_{T-t}(Y_t)  \d B_t \label{eq:qt-sde}
\end{align}
with $(B_t)$ the standard Brownian motion in $\mathbb{R}^d$, 
which can be derived using It\^{o}'s formula and can be found  in the proof of  \citet[Lemma~3]{benton2024nearly}. 
Recognizing that $\sigma^2_{T-t} = 1-e^{-2(T-t)}$, one can easily obtain
$
    \d (\sigma_{T-t}) = - ({1-\sigma^2_{T-t}}/{\sigma_{T-t}})\d t.
$
With equations~\eqref{eq:qt-sde} in mind, we arrive at the following SDE representation of $\mu_{T-t}(Y_t)$:
\begin{align}
    \d \big(\mu_{T-t}(Y_t)\big) &= \d \Bigg(\frac{1}{\sqrt{1-\sigma_{T-t}^2}}Y_t + \frac{\sigma_{T-t}^2}{\sqrt{1-\sigma_{T-t}^2}}R_t\Bigg) \notag \\
    &= e^{-t} \big(- Y_t\d t + \d Y_t\big) - \frac{2-\sigma_{T-t}^2}{\sqrt{1-\sigma_{T-t}^2}}R_t \d t + \frac{\sigma_{T-t}^2}{\sqrt{1-\sigma_{T-t}^2}}\d R_t \notag\\
    &= \frac{\sqrt{2}\sigma_{T-t}^2}{\sqrt{1-\sigma_{T-t}^2}}\left(\nabla s_{T-t}(Y_t) + \frac{1}{\sigma_{T-t}^2} I_D\right) \d B_t. \label{eq:d-mu}
\end{align}

Now, let us recall some useful properties for a general diffusion process $\d X_t = \mu(X_t,t)\d t + \sigma(X_t,t)\d B_t$, for which It\^{o}'s formula applied to $f(\cdot) = \|\cdot\|^2$ gives
\begin{align*}
    \frac{\d (\|X_t\|_2^2)}{\d t} =  2\langle \mu(X_t,t) , X_t \rangle + \|\sigma(X_t,t)\|_\F^2  .
\end{align*}
Thus, it follows that
\begin{align*}
    \frac{\d}{\d t}\E\big[\|X_t\|_2^2\big] 
    %
    = 2\E\left[ \langle \mu(X_t,t) , X_t \rangle\right] + \E\left[\|\sigma(X_t,t)\|_\F^2\right].
\end{align*}
Applying this general formula to the special case described in equation~\eqref{eq:d-mu}, we obtain 
\begin{align}
    &\frac{\d}{\d t} \E\left[\big\|\mu_{T-t}(Y_t) - \mu_{T - s}(Y_{s})\big\|_2^2\right] = \frac{2\sigma_{T-t}^4}{1-\sigma_{T-t}^2} \E\left[\left\|\nabla s_{T-t}(Y_t) + \frac{1}{\sigma_{T-t}^2} I_D\right\|_\F^2\right]. \label{eq:score-expectation-grad}
\end{align}

To further simplify the expression above, we resort to the following lemma that provides a useful characterization of the score function; see \citet[Lemma~5]{benton2024nearly} for the proof.  
\begin{lemma}\label{lemma:score-grad}
The score function satisfies 
    \begin{align*}
        \nabla s_{T-t}(Y_t) = -\frac{1}{\sigma_{T-t}^2}I_d + \frac{1-\sigma_{T-t}^2}{\sigma_{T-t}^4} \Sigma_{T-t}(Y_t).
    \end{align*}
\end{lemma}
\noindent 
Taking Lemma~\ref{lemma:score-grad} collectively with equation~\eqref{eq:score-expectation-grad}, we can deduce that
\begin{align}
    \frac{\d}{\d t} \E\left[\big\|\mu_{T-t}(Y_t) - \mu_{T - s}(Y_{s})\big\|_2^2\right] = 2\frac{1-\sigma^2_{T-t}}{\sigma^4_{T-t}}\E\left[\|\Sigma_{T-t}(Y_t)\|_\F^2\right].
    \label{eq:derivative-123}
\end{align}

Now, it remains to 
cope with the Frobenius norm of the posterior covariance in \eqref{eq:derivative-123}. 
To do so, we adopt the following result introduced in  \citet[Lemma 1]{benton2024nearly}, a crucial property borrowed from the stochastic localization literature \citep{eldan2020taming,el2022information}.
\begin{lemma}\label{lemma:sl}
    For all $t \in [0,T)$, one has $$\E\left[\|\Sigma_{T-t}(Y_t)\|_\F^2\right] = -\frac{\sigma_{T-t}^4}{2(1-\sigma_{T-t}^2)}\frac{\d}{\d t} \E\left[\tr\big(\Sigma_{T-t}(Y_t)\big)\right].$$
\end{lemma}
Combining all the results above, we arrive at
\begin{align*}
    D_{s,t} &= \frac{1-\sigma^2_{T-t}}{\sigma_{T-t}^4} \E\left[\big\|\mu_{T-t}(Y_t) - \mu_{T - s}(Y_{s})\big\|_2^2\right]\\
    &= \frac{2(1-\sigma^2_{T-t})}{\sigma_{T-t}^4}\int_s^t \frac{1-\sigma^2_{T-u}}{\sigma^4_{T-u}}\E\big[\big\|\Sigma_{T-u}(Y_u)\big\|_\F^2\big] \d u\\
    &\overset{\mathrm{(a)}}{=} -\frac{1-\sigma^2_{T-t}}{\sigma_{T-t}^4} \int_s^t \d \Big(\E\left[\tr\big(\Sigma_{T-u}(Y_u)\big)\right]\Big)\\
    &= \frac{1-\sigma^2_{T-t}}{\sigma_{T-t}^4} \Big(\E\left[\tr\big(\Sigma_{T-s}(Y_s)\big)\right] - \E\left[\tr\big(\Sigma_{T-t}(Y_t)\big)\right]\Big),
\end{align*}
where equality (a) applies Lemma~\ref{lemma:sl}. We have thus established the advertised result in Lemma~\ref{lemma:dst}. 

\subsection{Proof of Lemma~\ref{lemma:pos-var}}
\label{sec:proof:lemma:pos-var}
Recalling the fact that $(Y_t)_{t \in [0,T]} \overset{\mathrm{d}}{=} (X_{T-t})_{t\in [0,T]}$ (see \eqref{eq:equiv-Y-X-reverse}), we have
\begin{align}
    \E\left[\tr\big(\Sigma_{T-t}(Y_t)\big)\right] = \E\left[\tr\big(\Sigma_{T-t}(X_{T-t})\big)\right].\label{eq:x-y}
\end{align}
In view of this relation, we are allowed to prove Lemma~\ref{lemma:pos-var} regarding the forward process $(X_t)_{t\in [0,T]}$.

By the monotonicity of $\E\left[\tr\big(\Sigma_{T-t}(Y_t)\big)\right]$ stated in \eqref{eq:non-increasing-Sigma}, we know that
\begin{align}
    \E\left[\tr\big(\Sigma_{T-t}(Y_t)\big)\right] &\le \E\left[\tr\big(\Sigma_{T}(Y_0)\big)\right] =  \E\left[\tr\big(\Sigma_{T}(X_{T})\big)\right] \notag\\
& \le \E_{X_T}\big[\E[\|X_0\|^2|X_T]\big] \le \E[\|X_0\|^2],\label{eq:trace-bound-1}
\end{align}
where the last line has used the definition of $\Sigma_{T}$. 
Thus, everything comes down to showing that
\begin{align*}
    \E\left[\tr\big(\Sigma_{T-t}(Y_t)\big)\right]  \lesssim \frac{\sigma^2_{T-t}}{1-\sigma^2_{T-t}}k\log k, 
    \qquad \forall t\in [0,T),
\end{align*}
which, according to equation~\eqref{eq:x-y},  is equivalent to proving that
\begin{align}
    \E\left[\tr\big(\Sigma_{t}(X_t)\big)\right]  \lesssim \frac{\sigma^2_{t}}{1-\sigma^2_{t}}k\log k,
    \qquad \forall t\in [0,T).\label{eq:key-lemma}
\end{align}
In the sequel, we would like to prove inequality~\eqref{eq:key-lemma} using Assumption~\ref{ass:low-dim} as well as a covering argument.

Define 
$$
s = N^{\mathsf{cover}}(\mathcal{X}_{\mathsf{data}},\|\cdot\|_2,\varepsilon_0) \le \exp(C_{\mathsf{cover}}' k \log k),$$
where $C_{\mathsf{cover}}'>0$ is a universal constant chosen as $C_{\mathsf{cover}}' = C_0  C_{\mathsf{cover}}$ (with $C_0$ specified in Assumption~\ref{ass:low-dim}).
Let $\{x_i^*\}_{1\le i \le s}$ be an $\varepsilon_0$-net of the data support $\mathcal{X}_{\mathsf{data}}$ (see the definition of an epsilon-net in \citet{wainwright2019high}), and let $\{\B_i\}_{1\le i \le s}$ be the corresponding $\varepsilon_0$-covering of $\mathcal{X}_{\mathsf{data}}$ obeying $x_i^* \in \B_i$ for all $1\leq i\leq s$. 
Without loss of generality, we choose the $\varepsilon_0$-covering in a way such that $\B_i$ and $\B_j$ are disjoint for all $1 \le i \neq j \le s$, 
and $B(x_i^*, \varepsilon_0/2) \subset \B_i$ for all $1 \le i \le s$.

To make our analysis more clear, we find it helpful to define the following sets:
\begin{subequations}
\begin{align}
    \mathcal{I} &\coloneqq \big\{1\le i \le s \mid \P(X_0 \in \B_i) \ge \exp(-C_6k \log k)\big\},\label{eq:i}\\
    \mathcal{G} &\coloneqq \big\{z \in \R^d \mid \|z\|_2 \le 2\sqrt{d} + \sqrt{C_6 k \log k} \quad \text{and}\notag\\
    &\qquad \qquad \qquad \qquad \langle x_i^* - x_j^*, z\rangle \le \sqrt{C_6 k \log k}\,\|x_i^* - x_j^*\|_2  \quad \text{for all } 1\le i, j \le s\big\},\label{eq:g}
\end{align}
\end{subequations}
where $C_6>0$ is some sufficiently large constant.
The index set $\mathcal{I}$ can be regarded as the collection of sets in $\varepsilon_0$-covering with at least reasonably large probability, and, as we will see momentarily, $\mathcal{G}$ represents a high-probability region of a standard Gaussian random vector.
Further, define a high-probability region of the random vector $X_t$ as:
\begin{align}
    \T_t \coloneqq \left\{\sqrt{1-\sigma_t^2} x_0 + \sigma_t z \,\Big|\, x_0 \in \bigcup_{i\in \mathcal{I}}\B_i, z \in \mathcal{G}\right\},\notag
\end{align}
whose high-probability property is confirmed by the following lemma.  
The proof of this lemma is provided in  Appendix~\ref{sec:proof:lemma:high-prob-xt}. 
\begin{lemma}\label{lemma:high-prob-xt}
    For all $0 < t \le T$, it holds that 
    \begin{align*}
        \P(X_t \notin \T_t) \le \exp\left(-\frac{C_6}{4}k \log k\right). 
    \end{align*}
\end{lemma}

Next, to obtain tight control of $\E\left[\tr\big(\Sigma_{t}(X_t)\big)\right]$, we look at the events $\{X_t \in \T_t\}$ and $\{X_t \notin \T_t\}$ separately.

\paragraph{(i) The case when $X_t \in \T_t$.} For this case, the following lemma characterizes the typical performance of $X_0$, whose proof can be found in Appendix~\ref{sec:proof:lemma:high-prob-post}.
\begin{lemma}\label{lemma:high-prob-post}
    Consider any $x \in \T_t$, and suppose that $x = \sqrt{1-\sigma_t^2} x_0' + \sigma_t z$ with $x_0' \in \B_{i(x)}$ and $z \in \mathcal{G}$. Then
    \begin{align*}
        \P\left(\big\|X_0 - x_{i(x)}^{*}\big\|_2 \ge \frac{\sigma_t}{\sqrt{1-\sigma_t^2}} \cdot \sqrt{C k\log k} \,\Big|\, X_t = x\right) \le \exp\left(-\frac{C}{20}k\log k\right)
    \end{align*}
    holds for any  $C \ge C_7$,  where $C_7>0$ is some large enough constant.
\end{lemma}
Armed with the sub-Gaussian tail bound in Lemma~\ref{lemma:high-prob-post}, one can directly check that, for any $x \in \T_t$,
\begin{align}
 \tr\big(\Sigma_{t}(x)\big)\le\E\big[\big\| X_{0}-x_{i(x)}^{*}\big\|_{2}^{2}\mid X_{t}=x\big]\lesssim\frac{\sigma_{t}^{2}}{1-\sigma_{t}^{2}}k\log k
    \label{eq:high-prob-sigma}
\end{align}
holds with $i(x)$ specified in Lemma~\ref{lemma:high-prob-post}.

\paragraph{(ii) The case when $X_t \notin \T_t$.} In this case, we can simply recall Assumption~\ref{ass:bounded} to derive
    \begin{align}
        \tr\big(\Sigma_t(x)\big) \le \E[\|X_0\|_2^2 \mid X_t = x] \le \sup_{x \in \mathcal{X}_{\mathsf{data}}} \|x\|_2^2 \le k^{2C_R}. \label{eq:simple-bound}
    \end{align}

\paragraph{(iii) Putting all this together.}
Combining inequalities~\eqref{eq:high-prob-sigma} and \eqref{eq:simple-bound} with Lemma~\ref{lemma:high-prob-xt}, we arrive at
\begin{align}
    \E\left[\tr\big(\Sigma_{t}(X_t)\big)\right] &
    \lesssim \P(X_t \in \T_t) \cdot
    \frac{\sigma_t^2}{1-\sigma_t^2} k \log k + \P(X_t \notin \T_t) \cdot k^{2C_R}
    \notag\\
    &\lesssim \frac{\sigma_t^2}{1-\sigma_t^2} k \log k + \exp\left(-\frac{C_6}{4}k \log k\right) k^{2C_R}\notag\\
    &\le \frac{\sigma_t^2}{1-\sigma_t^2} k \log k + \exp\left(-\frac{C_6}{5}k \log k\right)\notag\\
    &\asymp \frac{\sigma_t^2}{1-\sigma_t^2} k \log k, \label{eq:trace-bound-2}
\end{align}
provided that $C_6$ is sufficiently large and $\sigma_t^2/(1-\sigma_t^2) 
\geq \sigma_{\delta}^2 \asymp \min \{\delta, 1\} = \Omega(1/\mathrm{poly}(k))$.
Inequality~\eqref{eq:trace-bound-2} taken together with inequality~\eqref{eq:trace-bound-1} concludes the proof.

\section{Proofs of the auxiliary lemmas}
\label{sec:pf-supp-lems}

\subsection{Proof of Lemma~\ref{prop:manifold}}
\label{sec:proof-prop:manifold}

In this proof, we are in need of the following lemma concerning the covering number of a compact manifold.
\begin{lemma}\label{lemma:manifold-covering}
    For any $0 < \varepsilon_0 \le \tau_{\M}$, the covering number of a compact manifold satisfies 
    \begin{align*}
        N^{\mathsf{cover}}(\M, \|\cdot\|_2, \varepsilon_0) \le \frac{\mathsf{vol}(\M)}{\omega_k} \left(\frac{\pi}{2}\right)^k \varepsilon_0^{-k},
    \end{align*}
    where $\omega_k$ is the volume of the unit ball in $\R^k$.
\end{lemma}
The proof of Lemma~\ref{lemma:manifold-covering} can be found in \citet[Appendix A]{block2022intrinsic}. It is worth noting that in the proof, the $\varepsilon_0$-covering is constructed as $\M \subseteq \bigcup_{i=1}^s B(x_i,\varepsilon_0),$ where $s = N^{\mathsf{cover}}(\M, \|\cdot\|_2, \varepsilon_0)$ and $x_i \in \M$ for all $i=1,\ldots,s$.
By direct calculation and Lemma~\ref{lemma:manifold-covering}, we have
\begin{align*}
N^{\mathsf{cover}}(\M^{\varepsilon_0}, \|\cdot\|_2, \varepsilon_0) \le N^{\mathsf{cover}}(\M, \|\cdot\|_2, \varepsilon_0/2) \le \frac{\mathsf{vol}(\M)}{\omega_k} \pi^k \varepsilon_0^{-k}.
\end{align*}
Thus, we obtain an upper bound on the metric entropy:
\begin{align*}
    \log N^{\mathsf{cover}}(\M^{\varepsilon_0}, \|\cdot\|_2, \varepsilon_0) &\le k\log \varepsilon_0^{-1} + \log \mathsf{vol}(\M) - \log \omega_k + k \log \pi\\
    &\le k \left(\log \varepsilon_0^{-1} + (\mathsf{vol}(\M))^{1/k} + C_4 \right),
\end{align*}
where $C_4$ is some numerical constant. 
Therefore, if $\M$ satisfies $\tau_\M = \Omega(k^{-C_1})$ and $\mathsf{vol}(\M) = O(k^{C_2})$ for some large enough constants $C_1,C_2>0$, we arrive at  
\begin{align*}
    \log N^{\mathsf{cover}}(\M^{\varepsilon_0}, \|\cdot\|_2, \varepsilon_0) \le (C_0 + C_1 + C_2) k \log k \asymp k \log k,
\end{align*} 
which verifies  Assumption~\ref{ass:low-dim}.

\subsection{Proof of Lemma~\ref{lemma:high-prob-xt}}
\label{sec:proof:lemma:high-prob-xt}
Since $\T_t$ is characterized through two different sets $\bigcup_{i \in \mathcal{I}}\B_i$ and $\mathcal{G}$, we analyze them separately.

Let us begin with the set $\bigcup_{i \in \mathcal{I}}\B_i$, which satisfies 
    \begin{align}
        \P\bigg(X_0 \notin \bigcup_{i \in \mathcal{I}}\B_i\bigg) &\overset{\mathrm{(a)}}{\le} \sum_{i \in \mathcal{I}} \exp(-C_6 k \log k)\notag\\
        &\le \exp\big( - (C_6 - C_{\mathsf{cover}}') k \log k\big)\notag\\
        &\overset{\mathrm{(b)}}{\le} \frac{1}{2} \exp\left(-\frac{C_6}{4} k\log k\right).\label{eq:b-bound}
    \end{align}
    Here, we apply  the union bound in inequality (a), and rely on the fact that $C_6$ is large enough (e.g., $C_6 > 8 C_{\mathsf{cover}}'$) in equality (b).

 Regarding the set $\mathcal{G}$, we analyze it in a similar manner.  First,  recall some standard concentration bounds below (see, e.g.,  \citet[Chapter 2]{wainwright2019high}).
\begin{lemma}\label{lemma:tail-bound}
        For $Z \sim \mathcal{N}(0, I_d)$, $Z_1$ as the first coordinate of $Z$ and any $t > 1$, we have the following upper bound on the tail probability of Gaussian and chi-square random variables:
        \begin{align*}
            &\P(|Z_1| \ge t) \le e^{-t^2/2}
            \qquad \text{and}\qquad
            \P(\|Z\|_2^2 \ge \sqrt{d} + t) \le e^{-t^2/2}.
        \end{align*}
        %
    \end{lemma}
    By virtue of Lemma~\ref{lemma:tail-bound}, we can demonstrate that
    \begin{align}
        \P(Z \notin \mathcal{G}) &\overset{\mathrm{(a)}}{\le} \P\left(\|Z\|_2 > \sqrt{d} + \sqrt{C_6 k \log k}\right) + \sum_{i=1}^s\sum_{j=1}^s \P\left(\langle x_i^* - x_j^*, Z\rangle > \sqrt{C_6 k \log k}\,\|x_i^* - x_j^*\|_2 \right)\notag\\
        &\le (s^2 + 1) \exp\left(-\frac{C_6}{2} k \log k\right) \notag\\
        &\le \big(\exp(2C_{\mathsf{cover}}'k\log k) + 1\big) \exp\left(-\frac{C_6}{2} k \log k\right)\notag\\
        &\overset{\mathrm{(b)}}{\le} \frac{1}{2} \exp\left(-\frac{C_6}{4} k\log k\right),\label{eq:g-bound}
    \end{align}
    where (a) applies the union bound, and (b) holds for large enough  $C_6$  (e.g., $C_6 > 8 C_{\mathsf{cover}}'$).

Recalling that $X_t \overset{\mathrm{d}}{=} \sqrt{1-\sigma_t^2}X_0 + \sigma_t Z$, one can combine inequalities~\eqref{eq:b-bound} and \eqref{eq:g-bound} to yield
    \begin{align*}
         \P(X_t \notin \T_t) \le \P\left(X_0 \notin \bigcup_{i \in \mathcal{I}}\B_i\right) + \P(Z \notin \mathcal{G}) \le \exp\left(-\frac{C_6}{4}k \log k\right).
    \end{align*}

\subsection{Proof of Lemma~\ref{lemma:high-prob-post}} 
\label{sec:proof:lemma:high-prob-post}
For ease of presentation, we find  it convenient to define the following set
\begin{align}
    \mathcal{E}_t(x) \coloneqq \left\{x_0 \,\Big|\, \big\|x_0 - x_{i(x)}^{*}\big\|_2 \ge \frac{\sigma_t}{\sqrt{1-\sigma_t^2}} \cdot \sqrt{C k\log k}\right\},
    \label{eq:defn-Et-x}
\end{align}
where $C_7>0$ is some sufficiently large constant (to be specified momentarily) and $C > C_7$. 
It thus boils down to bounding  $\P(\mathcal{E}_t(x) \mid X_t =x)$.
Denote by $q_{X_t|X_0}(\cdot|x_0)$  the conditional density function of $X_t$ given $X_0 = x$, and $q_0$  the measure induced by $X_0$.
Making use of the Bayes formula, we obtain
\begin{align}
    \P(\mathcal{E}_t(x) \mid X_t =x) &= \frac{\int_{\mathcal{E}_t(x)}q_{X_t|X_0}(x\,|\,x_0) \d q_0(x_0)}{\int_{\mathcal{X}_{\mathsf{data}}}q_{X_t|X_0}(x\,|\,\widetilde{x}_0) \d q_0(\widetilde{x}_0)} \notag\\
    &\le \frac{\int_{\mathcal{E}_t(x)}q_{X_t|X_0}(x|x_0) \d q_0(x_0)}{\int_{\B_{i(x)}}q_{X_t|X_0}(x\,|\,\widetilde{x}_0) \d q_0(\widetilde{x}_0)} \notag\\
    &\le \frac{\int_{\mathcal{E}_t(x)}q_{X_t|X_0}(x\,|\,x_0) \d q_0(x_0)}{\P_{X_0}(\B_{i(x)})\inf_{\widetilde{x}_0 \in \B_{i(x)}}q_{X_t|X_0}(x\,|\,\widetilde{x}_0)}\notag\\
    &\le \frac{1}{\P_{X_0}(\B_{i(x)})} \cdot \frac{\sup_{x_0 \in \mathcal{E}_t(x)}q_{X_t|X_0}(x\,|\,x_0)}{\inf_{\widetilde{x}_0 \in \B_{i(x)}}q_{X_t|X_0}(x\,|\,\widetilde{x}_0)}\notag\\
    &\le \exp\left(C_6 k\log k\right) \cdot \sup_{x_0 \in \mathcal{E}_t(x), \widetilde{x}_0 \in \B_{i(x)}} \exp \left\{\frac{1}{2\sigma_t^2}\left(\Big\|x - \sqrt{1-\sigma_t^2}\widetilde{x}_0\Big\|_2^2 - \Big\|x - \sqrt{1-\sigma_t^2}x_0\Big\|_2^2\right)\right\},\label{eq:etx-bound}
\end{align}
where the last inequality holds by exploiting $i(x) \in \mathcal{I}$ and the definition of $\mathcal{I}$ in \eqref{eq:i}.
It then suffices to upper bound the exponential term above. 

Given that $x \in \T_t$, we can decompose it as $x = \sqrt{1-\sigma_t^2}x_0' + \sigma_t z$ for some $x_0' \in \B_{i(x)}$ and $z \in \mathcal{G}$.
For any $x_0 \in \mathcal{E}_t(x)$ and $\widetilde{x}_0 \in \B_{i(x)}$, suppose that $x_0 \in \B_j$, then 
direct calculation gives
\begin{align}
 & \Big\| x-\sqrt{1-\sigma_{t}^{2}}\widetilde{x}_{0}\Big\|_{2}^{2}-\Big\| x-\sqrt{1-\sigma_{t}^{2}}x_{0}\Big\|_{2}^{2}\notag\\
 & \qquad=-(1-\sigma_{t}^{2})\|x_{0}'-x_{0}\|_{2}^{2}+2\sqrt{\sigma_{t}^{2}(1-\sigma_{t}^{2})}\,\langle x_{0}-\widetilde{x}_{0},z\rangle+(1-\sigma_{t}^{2})\|x_{0}'-\widetilde{x}_{0}\|_{2}^{2}\notag\\
 & \qquad\overset{\mathrm{(a)}}{\le}-(1-\sigma_{t}^{2})\left(\big\| x_{i(x)}^{*}-x_{j}^{*}\big\|_{2}-2\varepsilon_{0}\right)^{2}+2\sqrt{\sigma_{t}^{2}(1-\sigma_{t}^{2})}\langle x_{0}-\widetilde{x}_{0},z\rangle+4(1-\sigma_{t}^{2})\varepsilon_{0}^{2}\notag\\
 & \qquad=-(1-\sigma_{t}^{2})\big\| x_{i(x)}^{*}-x_{j}^{*}\big\|_{2}^{2}+4(1-\sigma_{t}^{2})\big\| x_{i(x)}^{*}-x_{j}^{*}\big\|_{2}\varepsilon_{0}\notag\\
 & \qquad\qquad\qquad\qquad\qquad+2\sqrt{\sigma_{t}^{2}(1-\sigma_{t}^{2})}\left(\big\langle x_{j}^{*}-x_{i(x)}^{*},z\big\rangle+\big\langle x_{0}-x_{j}^{*},z\big\rangle-\big\langle\widetilde{x}_{0}-x_{i(x)}^{*},z\big\rangle\right)\notag\\
 & \qquad\overset{\mathrm{(b)}}{\le}-(1-\sigma_{t}^{2})\big\| x_{i(x)}^{*}-x_{j}^{*}\big\|_{2}^{2}+4(1-\sigma_{t}^{2})\big\| x_{i(x)}^{*}-x_{j}^{*}\big\|_{2}\varepsilon_{0}\notag\\
 & \qquad\qquad\qquad\qquad\qquad+2\sqrt{\sigma_{t}^{2}(1-\sigma_{t}^{2})}\left(\big\langle x_{j}^{*}-x_{i(x)}^{*},z\big\rangle+2\left(\sqrt{d}+\sqrt{C_{6}k\log k}\right)\varepsilon_{0}\right),
\label{eq:exp-bound}
\end{align}
%
where inequality (a) holds due to the property of the $\varepsilon_0$-net and the triangle inequality, and inequality (b) comes from the definition of $z \in \mathcal{G}$ as well as Cauchy-Schwarz.
To further bound \eqref{eq:exp-bound}, recalling the definition of $\mathcal{G}$ again gives 
\begin{align}
    \big\langle x_j^* - x_{i(x)}^*, z\big\rangle \le \sqrt{C_6k \log k} \,\big\|x_j^* - x_{i(x)}^*\big\|_2.\label{eq:inner=prod-bound}
\end{align}
In addition, for a sufficiently large constant $C_0>0$, it is readily seen that
\begin{align}
    \varepsilon_0 = k^{-C_0} \le \sigma_\delta \min\left\{1, \sqrt{\frac{k \log k}{d}}\right\} \le \frac{\sigma_t}{\sqrt{1-\sigma_t^2}} \min\left\{1, \sqrt{\frac{k \log k}{d}}\right\},\label{eq:ep0-t-bound}
\end{align}
where we recall that $d\lesssim \mathrm{poly}(k)$, $ \delta=\Omega(1/\mathrm{poly}(k))$ and $\sigma_{\delta}=\sqrt{1-e^{-2\delta}}\asymp \min\{\delta, 1\}$. 
Moreover, 
\begin{align}
    4\sqrt{\sigma_t^2(1-\sigma_t^2)}\left(\sqrt{d} + \sqrt{C_6k \log k}\right)\varepsilon_0 \le 5 C_6\sigma_t^2 k\log k\label{eq:mid-bound}
\end{align}
for large enough $C_6>0$. 
Substituting \eqref{eq:inner=prod-bound}, \eqref{eq:ep0-t-bound} and \eqref{eq:mid-bound} into \eqref{eq:exp-bound}, and rearranging terms, we obtain
\begin{align}
    &\Big\|x - \sqrt{1-\sigma_t^2}\widetilde{x}_0\Big\|^2_2 - \Big\|x - \sqrt{1-\sigma_t^2}x_0\Big\|_2^2\notag\\
    &\le -(1-\sigma_{t}^{2})\big\| x_{i(x)}^{*}-x_{j}^{*}\big\|_{2}^{2} + \left(4(1-\sigma_{t}^{2})\varepsilon_{0} + 2\sqrt{C_6 \sigma_t^2 (1-\sigma_t^2)k \log k}\right)\big\| x_{i(x)}^{*}-x_{j}^{*}\big\|_{2}+ 5C_6\sigma_t^2 k \log k\notag\\ 
    &\overset{\mathrm{(a)}}{\le} - \frac{1-\sigma_t^2}{2} \big\|x_{i(x)}^{*} - x_j^*\big\|_2^2 + 5C_6\sigma_t^2 k \log k\notag\\
    &\overset{\mathrm{(b)}}{\le} -\frac{1-\sigma_t^2}{8}\big\|x_{i(x)}^{*} - x_0\big\|_2^2 + 5C_6\sigma_t^2 k \log k\notag\\
    &\overset{\mathrm{(c)}}{\le} - \frac{C}{8}\sigma_t^2 k\log k + 5C_6\sigma_t^2 k \log k\notag\\
    &\overset{\mathrm{(d)}}{\le} - \frac{C}{9}\sigma_t^2 k \log k,\notag
\end{align}
where (a) is from the definition~\eqref{eq:defn-Et-x} and the fact that $C_0$ and $C$ are large enough, 
(b) results from $\big\|x_{i(x)}^{*} - x_j^*\big\|_2 \ge  \big\|x_{i(x)}^{*} - x_0\big\|_2/2$, a basic property of the construction of the covering, 
(c) makes use of the definition of $\mathcal{E}_t$ in \eqref{eq:defn-Et-x}, 
and (d) holds provided that $C$ is large enough.
Taking the above bound collectively with inequality~\eqref{eq:etx-bound} yields
\begin{align*}
    \P(\mathcal{E}_t(x) \mid X_t =x) &\le \exp\left(C_6 k\log k\right)\cdot \sup_{x_0 \in \mathcal{E}_t(x), \widetilde{x}_0 \in \B_{i(x)}} \exp \left\{\frac{1}{2\sigma_t^2}\left(\Big\|x - \sqrt{1-\sigma_t^2}\widetilde{x}_0\Big\|_2^2 - \Big\|x - \sqrt{1-\sigma_t^2}x_0\Big\|_2^2\right)\right\}  \notag\\ 
    &\le \exp\left(C_6 k\log k\right) \cdot \exp\left(-\frac{C}{18}k\log k\right) \le \exp\left(-\frac{C}{20}k\log k\right),
\end{align*}
thus concluding the proof of this lemma.


\bibliographystyle{apalike}
\bibliography{reference-notes,reference-diffusion}

\end{document}